\documentclass{article}

\usepackage{microtype}
\usepackage{graphicx}
\usepackage{subfigure}
\usepackage{booktabs} 

\usepackage{hyperref}

\newcommand{\RRR}[1]{{\color{black}{\sc}#1}}
\newcommand{\RRRR}[1]{{\color{black}{\sc}#1}}

\usepackage[accepted]{icml2024}

\usepackage{amsmath}
\usepackage{amssymb}
\usepackage{mathtools}
\usepackage{amsthm,multirow,bm}
\definecolor{pink}{rgb}{1,0.94,1}
\usepackage{colortbl}
\usepackage{enumitem}
\def\method{PGODE}
\newcommand{\RR}[1]{{\color{black}{\sc}#1}}
\usepackage[capitalize,noabbrev]{cleveref}

\theoremstyle{plain}
\newtheorem{theorem}{Theorem}[section]

\newtheorem{lemma}[theorem]{Lemma}

\theoremstyle{definition}

\theoremstyle{remark}

\usepackage[textsize=tiny]{todonotes}

\icmltitlerunning{PGODE: Towards High-quality System Dynamics Modeling}

\begin{document}

\twocolumn[
\icmltitle{PGODE: Towards High-quality System Dynamics Modeling}


\icmlsetsymbol{equal}{*}

\begin{icmlauthorlist}
\icmlauthor{Xiao Luo}{a}
\icmlauthor{Yiyang Gu}{b}
\icmlauthor{Huiyu Jiang}{c}
\icmlauthor{Hang Zhou}{d}
\icmlauthor{Jinsheng Huang}{b}
\icmlauthor{Wei Ju}{b}\\
\icmlauthor{Zhiping Xiao}{a}
\icmlauthor{Ming Zhang}{b}
\icmlauthor{Yizhou Sun}{a}
\end{icmlauthorlist}

\icmlaffiliation{a}{Department of Computer Science, University of California, Los Angeles, USA}
\icmlaffiliation{b}{National Key Laboratory for Multimedia Information Processing, School of Computer Science, Peking University}
\icmlaffiliation{c}{Department of Statistics and Applied Probability, University of California, Santa Barbara, USA}
\icmlaffiliation{d}{Department of Statistics, University of California, Davis, USA}

\icmlcorrespondingauthor{Xiao Luo}{xiaoluo@cs.ucla.edu}

\icmlkeywords{Machine Learning, ICML}

\vskip 0.3in
]



\printAffiliationsAndNotice{} 

\begin{abstract}
This paper studies the problem of modeling multi-agent dynamical systems, where agents could interact mutually to influence their behaviors. Recent research predominantly uses geometric graphs to depict these mutual interactions, which are then captured by powerful graph neural networks (GNNs). However, predicting interacting dynamics in challenging scenarios such as out-of-distribution shift and complicated underlying rules remains unsolved. In this paper, we propose a new approach named \RR{\underline{P}rototypical \underline{G}raph \underline{ODE} (\method{})} to address the problem. The core of \method{} is to incorporate prototype decomposition from contextual knowledge into a \textit{continuous} graph ODE framework. Specifically, \method{} employs representation disentanglement and system parameters to extract both object-level and system-level contexts from historical trajectories, which allows us to explicitly model their independent influence and thus enhances the \textit{generalization} capability under system changes. Then, we integrate these disentangled latent representations into a graph ODE model, which determines a combination of various interacting prototypes for enhanced model \textit{expressivity}. The entire model is optimized using an end-to-end variational inference framework to maximize the likelihood. Extensive experiments in both in-distribution and out-of-distribution settings validate the superiority of \method{} compared to various baselines. 

\end{abstract}

\vspace{-0.5cm}
\section{Introduction}

Multi-agent dynamical systems~\cite{huang2023neuralstagger} are ubiquitous in the real world where agents can be vehicles~\cite{yildiz2022learning} and microcosmic particles~\cite{shao2022transformer}. These agents could have complicated interactions resulting from behavioral or mechanical influences, which result in complicated future trajectories of the whole system. \RRR{Modeling the interacting dynamics is a crucial challenge in machine learning with broad applications in fluid mechanics~\cite{pfaff2020learning,mayr2023boundary}, autonomous driving~\cite{yu2020spatio,zhu2023biff}, and molecular dynamics~\cite{wu2024equivariant,xu2023eqmotion}.} Extensive time-series approaches based on recurrent neural networks~\cite{weerakody2021review} and Transformers~\cite{zhou2021informer,chen2023contiformer,chen2024multi} are generally designed for single-agent dynamical systems~\cite{fotiadis2023disentangled}, which fall short when it comes to capturing the intricate relationships among interacting objects. To address this gap, geometric graphs~\cite{kofinas2021roto} are usually employed to represent the interactions between objects where nodes represent individual objects, and edges are built when a connection exists between two nodes. These connections can be obtained from geographical distances between atoms in molecular dynamics~\cite{li2022graph} and underlying equations in mechanical systems~\cite{huang2020learning}.

In the literature, graph neural networks (GNNs)~\cite{kipf2016semi,xu2019powerful,zheng2022graph,li2022finding,he2022convolutional} have been increasingly prevailing for learning from geometric graphs in interacting dynamical systems~\cite{pfaff2020learning,shao2022transformer,sanchez2020learning,han2022learning,meirom2021controlling,yildiz2022learning}. These GNN-based approaches primarily focus on predicting the future states of dynamic systems with the message passing mechanism. Specifically, they begin with encoding the states of trajectories and then iteratively update each node representation by incorporating information from its adjacent nodes, which effectively captures the complex interacting dynamics among the objects in systems.

Despite the significant advancements, GNN-based approaches often suffer from performance decreasement in challenging scenarios such as long-term dynamics~\cite{lippe2023pde}, complicated governing rules~\cite{gu2022stochastic}, and out-of-distribution shift~\cite{dendorfer2021mg}. As a consequence, developing a high-quality data-driven model requires us to consider the following critical points: (1) \textit{Capturing Continuous Dynamics}. The majority of existing methods predict the whole trajectories in an autoregressive manner~\cite{pfaff2020learning,shao2022transformer,sanchez2020learning}, which iteratively feed next-time predictions back into the input. These rollouts could lead to error accumulation and thus fail to capture long-term dynamics accurately. (2) \textit{Expressivity}. There are a variety of interacting dynamical systems governed by complex partial differential equations (PDEs) in physics and biology~\cite{rao2023encoding,chen2023implicit}. Therefore, a high-quality model with strong \textit{expressivity} is anticipated for sufficient learning. (3) \textit{Generalization}. In practical applications, the distributions of training and test trajectories could differ due to variations in system parameters~\cite{sanchez2020learning,li2023transferable}. Current data-driven models could perform poorly when confronting system changes during the inference phase~\cite{goyal2022inductive}.

In this paper, we propose a novel approach named \RR{\underline{P}rototypical \underline{G}raph \underline{ODE} (\method{})} for complicated interacting dynamics modeling. \RR{The core of PGODE lies in exploring disentangled contexts, i.e., object states and system states, inferred from historical trajectories for graph ODE with high \textit{expressivity} and \textit{generalization}.} To begin, we extract both object-level and system-level contexts via message passing and attention mechanisms for subsequent dynamics modeling. Object-level contexts refer to individual attributes such as initial states and local heterophily~\cite{luan2022revisiting}, while system-level contexts refer to shared parameters such as temperature and viscosity. To improve generalization under system changes, we focus on two strategies. First, we enhance the invariance of object-level contexts under system changes through representation disentanglement. Second, we establish a connection between known system parameters and system-level latent representations. Furthermore, we incorporate this contextual information into a graph ODE framework to capture long-term dynamics through \textit{continuous} evolution instead of discrete rollouts. More importantly, we introduce a set of learnable GNN prototypes that can be trained to represent different interaction patterns. The weights for each object are then derived from its hierarchical representations to provide individualized dynamics. Our framework can be illustrated from a mixture-of-experts perspective, which boosts the \textit{expressivity} of the model. Finally, we integrate our method into an end-to-end variational inference framework to optimize the evidence lower bound (ELBO) of the likelihood. Comprehensive experiments in different settings validate the superiority of \method{} in comparison to state-of-the-art approaches.

The contributions of this paper can be summarized in three points: (1) \textit{New Connection.} To the best of our knowledge, this work is the first to connect context mining with a \RR{prototypical graph ODE approach} for modeling challenging interacting dynamics.
(2) \textit{Methodology}. We extract hierarchical contexts with representation disentanglement and system parameters, which are then integrated into a graph ODE model that utilizes \RR{prototype decomposition}.
(3) \textit{Superior Performance.} Extensive experiments validate the efficacy of our approach in different challenging settings.

\section{Background}

\textbf{Problem Definition.}
\RR{Given a multi-agent dynamical system, we characterize the agent states and interaction at the $t$-th timestamp as a graph $G^t = (\mathcal{V}, \mathcal{E}^t, \bm{X}^t)$, where each node in $\mathcal{V}$ is an object, $\mathcal{E}^t$ comprises all the edges and $\bm{X}^t$ is the object attribute matrix. $N$ represents the number of objects. Given the observations $G^{1:T_{obs}} = \{G^1, \cdots, G^{T_{obs}} \}$, our goal is to learn a model capable of predicting the future trajectories $\bm{X}^{T_{obs}+1: T}$. \RRRR{Our interacting dynamics system is governed by a set of equations with time-invariant system parameters denoted as $\bm{\xi}$.} Different values of parameters $\bm{\xi}$ could influence underlying dynamical principles, leading to potential shift in trajectory distributions. As a consequence, it is essential to extract contextual information related to both system parameters and node states from historical observations for high-quality future trajectory predictions.}

\begin{figure*}[!t]
    \centering
    \includegraphics[width=0.9\textwidth]{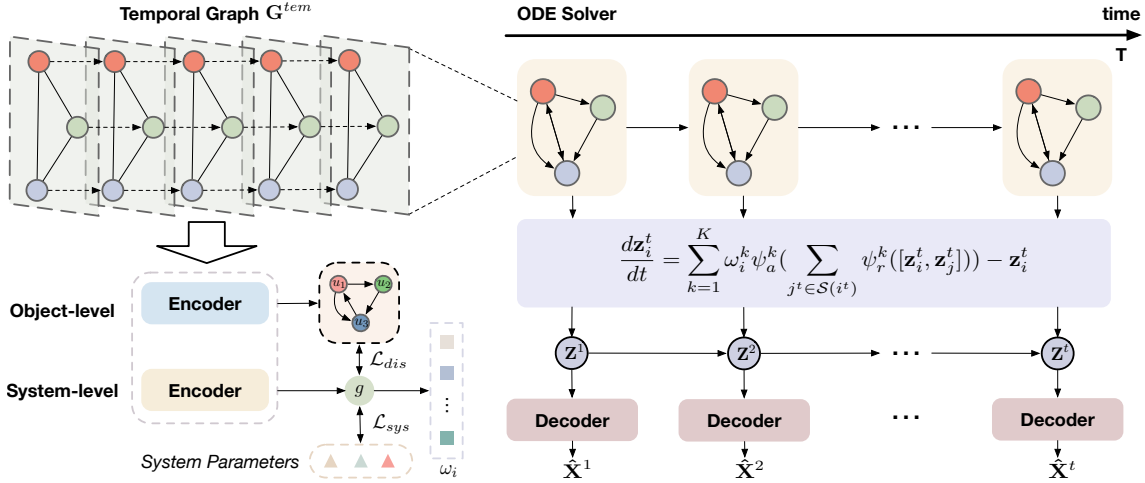}
    \vspace{-0.3cm}
    \caption{\RRRR{An overview of the proposed \method{}. Our \method{} first constructs a temporal graph and then utilizes different encoders to extract object-level and system-level contexts using representation disentanglement and system parameters. These contexts would generate weights for a prototypical graph ODE framework, which models the evolution of interacting objects. In the end, the latent states of objects are fed into a decoder to output the trajectories at any timestamp.}   
    }
    \label{fig:framework}
    \vspace{-0.4cm}
\end{figure*}

\textbf{Neural ODEs for Multi-agent Dynamical Systems.}
Neural ODEs have been shown effective in modeling various dynamical systems~\cite{chen2018neural,huang2021coupled,dupont2019augmented}. For single-agent dynamical systems, the evolution of latent representations $\bm{z}^t$ can be expressed via a given ODE $\frac{d\bm{z}^t}{dt} = f(\bm{z}^t)$. Then, the entire trajectory of the system can be determined using
$\bm{z}^{T}=\bm{z}^{0}+\int_{t=0}^{T} f\left(\bm{z}^{t}\right) d t$.
For multi-agent dynamical systems, the formulation can be extended as follows:
\begin{equation}\label{eq:multi-agent}
\bm{z}_{i}^{T}=\bm{z}_{i}^{0}+\int_{t=0}^{T} f_{i}\left(\bm{z}_{1}^{t}, \bm{z}_{2}^{t} \cdots \bm{z}_{N}^{t}\right) d t,
\end{equation}
where $\bm{z}_i^t$ represents the hidden embedding for the object $i$ at the timestamp $t$. $f_{i}$ models the interacting dynamics specifically for object $i$. With Eqn. \ref{eq:multi-agent}, we can calculate $\bm{z}_i^t$ using different numerical solvers including Runge-Kutta~\cite{schober2019probabilistic} and Leapfrog~\cite{zhuang2021mali}, which produce accurate predictions of future trajectories in the multi-agent systems using a decoder~\cite{luo2023hope}.

\section{The Proposed Approach}

This paper introduces a novel approach \method{} for modeling interacting system dynamics in challenging scenarios such as out-of-distribution shift and complicated underlying rules. The core of \method{} lies in exploring disentangled contexts for prototype decomposition for a high-quality graph ODE framework. Specifically, we first construct a temporal graph to learn disentangled object-level and system-level contexts from historical data and system parameters. These contexts further determine \RR{prototype decomposition}, which characterizes distinct interacting patterns in a graph ODE framework for modeling continuous dynamics. We adopt a decoder to output the trajectories and the whole model is optimized via an end-to-end variational inference framework. An overview of \method{} is depicted in Figure \ref{fig:framework}, and the details will be presented below. 

\subsection{Hierarchical Context Discovery with Disentanglement}

A promising solution to formulating the dynamics of interacting systems is the introduction of GNNs into Eqn. \ref{eq:multi-agent} where different GNNs are tailored for distinct nodes across diverse systems. \RRRR{Given the basic dynamical principles, the interacting dynamics of each object are influenced by both system-level and object-level contexts.} System-level contexts include temperature, viscosity, and coefficients in underlying equations~\cite{rama2017transition}, which are shared in the whole system. Object-level contexts refer to object attributes such as initial states, and local heterophily~\cite{luan2022revisiting}, which give rise to distinct interacting patterns for individual objects. To design GNNs for a variety of objects and system configurations, it is essential to derive object-level and system-level latent embeddings from historical trajectories. Additionally, system parameters could differ between training and test datasets~\cite{kim2021reversible}, thereby leading to potential distribution shift. To mitigate its influence, we disentangle object-level and system-level embeddings with known system parameters for a more precise and independent description of complex dynamical systems.

\textbf{Object-level Contexts.} We aim to condense the historical trajectories into informative object representations. Here, we conduct the message passing procedure on a temporal graph for observation representation updating. Then, object representations are generated by summarizing all the observations using the attention mechanism~\cite{niu2021review}. 

In detail, a temporal graph is first constructed where each node represents an observation~\cite{huang2021coupled}, and edges represent temporal and spatial relationships. Temporal edges connect successive observations of the same object, while spatial edges would be built when observations from two different objects are connected at the same timestamp. In formulation, we have $NT^{obs}$ nodes in the temporal graph $G^{tem}$ and its adjacency matrix can be written as:
\begin{equation}\label{eq:temporal_graph}
\bm{A}^{tem}(i^t,{j}^{t'})=\left\{\begin{array}{ll} w_{ij}^t & t=t', \\ 1 & i=j, t'=t+1,\\ 0 & \text { otherwise, }\end{array}\right.
\end{equation}
where $i^t$ represents the observation of $i$ at timestamp $t$ and $w_{ij}^t$ is the edge weight from $G^t$. Then, we adopt the message passing mechanism to learn from the temporal graph. Denote the representation of $i^t$ at the $l$-th layer as $\bm{h}_i^{t,(l)}$, and the interaction scores can be obtained by comparing representations between the query and key spaces as follows:
\RRRR{\begin{equation}
   \alpha^{(l)}(i^t, j^{t'}) = \frac{\bm{A}^{tem}(i^t, j^{t'})}{\sqrt{d}}  (\bm{W}_{query} \hat{\bm{h}}_i^{t,(l)})^T (\bm{W}_{key} \hat{\bm{h}}_j^{t',(l)}) ,
\end{equation}}
where $d$ denotes the hidden dimension and $\hat{\bm{h}}_i^{t,(l)} = {\bm{h}}_i^{t,(l)} + TE(t)$. \RRRR{$TE(t)$ is the temporal embedding with
$\mathrm{TE}(t)[2i]=\sin \left(\frac{t}{10000^{2 i / d}}\right)$ and $\mathrm{TE}(t)[2i+1]=\cos \left(\frac{t}{10000^{2 i / d}}\right)$, which provides the temporal information for our graph convolution module to capture temporal patterns and dependencies.}
$\bm{W}_{query} \in \mathbb{R}^{d\times d} $ and $\bm{W}_{key} \in \mathbb{R}^{d\times d}$ are two weight matrices for feature transformation. Then, we update each representation using its neighborhood as follows:
\begin{equation}
    \bm{h}_{i}^{t,(l+1)}=\bm{h}_{i}^{t,(l)}+\sigma\left(\sum_{j^{t'} \in \mathcal{S}({i^t})} \alpha^{(l)}(i^t, j^{t'}) \bm{W}_{value} \hat{\bm{h}}_{j}^{t',(l)} \right),
\end{equation}
where $\bm{W}_{value} \in \mathbb{R}^{d\times d}$ is to project representations into values and $\mathcal{S}(\cdot)$ collects all the neighboring nodes. In the end, we summarize all these observation representations for every object $i$ into a latent representation $\bm{u}_i$ using the attention mechanism as follows:
\begin{equation}\label{eq:local}
    \bm{q}_i^t = \bm{h}_i^{t,(L)} + \mathrm{TE}(t), \bm{u}_i = \frac{1}{N^{obs}}\sum_{t=1}^{N^{obs}} \sigma(\bm{W}_{sum} \bm{q}_i^t),
\end{equation}
in which $\bm{W}_{sum}$ is for feature transformation. In this manner, we incorporate semantics from both the observed trajectories and geometric structures into expressive object-level latent representations, i.e., $\{\bm{u}_i\}_{i=1}^N$ for predicting future complicated interacting dynamics in systems.

\textbf{System-level Contexts.} In real-world applications, system parameters may vary between training and test datasets, leading to out-of-distribution shift in trajectories~\cite{mirza2022efficient,ragab2023adatime}. To capture these variations and enhance model performance, we employ a separate network to infer system-level contexts from historical trajectories, which are guided by system parameters in the training data. Moreover, we employ mutual information minimization~\cite{sun2019infograph,feng2023towards} to disentangle object-level and system-level representations, which allows for a clear separation of influences and thus enables the invariance of object-level contexts under system changes.

In particular, we employ the same network architecture but with different parameters to generate the latent representation $\bm{u}'_i$ for object $i$. Then, a pooling operator is adopted to summarize all these object-level representations into a system-level representation $\bm{g}$ as $    \bm{g} = \sum_{i=1}^N \bm{u}'_i$.
To capture contexts from system parameters, we maximize the mutual information between the system-level representation and \RRRR{known parameters $\bm{\xi}$}, i.e., $I(\bm{g};\bm{\xi})$. Meanwhile, to disentangle object-level and system-level latent representation, we minimize their mutual information, i.e., $I(\bm{g};\bm{u}_i)$, which enables us to better handle the variations introduced by out-of-distribution system parameters. In our implementation, we make use of Jensen-Shannon mutual information estimator $T_\gamma(\cdot, \cdot)$~\cite{chen2019locality} with parameters $\gamma$, and the loss objective for learning system parameters can be:
\begin{equation}
\begin{aligned}
\mathcal{L}_{sys}&= \frac{1}{|\mathcal{P}|}\sum_{(\bm{g},\bm{\xi}) \in \mathcal{P}} -sp(-T_\gamma(\bm{g},\bm{\xi})) \\ &+ \frac{1}{|\mathcal{P}|^2} \sum_{(\bm{g},\bm{\xi}) \notin \mathcal{P}}sp(-T_\gamma(\bm{g},\bm{\xi})),
\end{aligned}
\end{equation}
where $sp(\bm{x})=\log(1+e^{\bm{x}})$ denotes the softplus function, \RRRR{$\bm{\xi}$ denotes the system parameters in dynamical systems, and $\mathcal{P}$ collects all the positive pairs from the same system.} Similarly, the loss objective for representation disentanglement is formulated as:
\begin{equation}
\begin{aligned}
    \mathcal{L}_{dis}&= max_{\gamma'} \{ \frac{1}{|\mathcal{P}'|}\sum_{(\bm{g},\bm{u}_i) \in \mathcal{P}'} sp(-T_{\gamma'}(\bm{g},\bm{u}_i)) \\ 
    &+ \frac{1}{|\mathcal{P}'| |\mathcal{P}| } \sum_{(\bm{g},\bm{u}_i) \notin \mathcal{P}'}-sp(-T_{\gamma'}(\bm{g},\bm{u}_i))\},
\end{aligned}
\end{equation}
where $T_{\gamma'}$ is optimization in an adversarial manner and $\mathcal{P}'$ collects all the positive object-system pairs. Differently, $T_{\gamma'}$ is trained adversarially for precise measurement of mutual information. On this basis, we establish the connection between system-level contexts and explicit parameters while simultaneously minimizing their impact on the object-level contexts through representation disentanglement. In this way, our model separates and accurately captures the influence of these two factors, enhancing the generalization capacity when system parameters vary during evaluation.

\subsection{\RR{Prototypical Graph ODE}}

After extracting context embeddings, we intend to integrate them into a graph ODE framework for multi-agent dynamic systems. However, training a separate GNN for each node would introduce an excessive number of parameters, which could result in overfitting and a complicated optimization process~\cite{zhao2020protoviewer,cini2023taming,guo2023newton}. To address this, we learn a set of GNN prototypes to characterize the entire GNN space, and then use \RR{prototype decomposition} for each object in the graph ODE. Specifically, we start by initializing state representations for each node and then determine the weights for each object based on both object-level and system-level contexts.

To begin, we utilize object-level contexts with feature transformation for initialization. Here, the initial state representations are sampled from an approximate posterior distribution $q(\bm{z}_i^0|G^{tem})$, which would be close to a prior distribution $p(\bm{z}_i^0)$. The mean and variance are learned from $\bm{u}_i$ as:
\begin{equation}
q\left(\bm{z}_{i}^{0} \mid G^{tem}\right)=\mathcal{N}\left(\psi^{m}\left(\bm{u}_{i}\right), \psi^{v}\left(\bm{u}_{i}\right)\right),
\end{equation}
where $\psi^{m}(\cdot)$ and $\psi^{v}(\cdot)$ are two feed-forward networks (FFNs) to compute the mean and variance. Then, we introduce $K$ GNN prototypes, each with two FFNs $\psi_r^k(\cdot)$ and $\psi_a^k(\cdot)$ for relation learning and feature aggregation, respectively. The updating rule of the $k$-th prototypes for object $i$ is formulated as follows:
\begin{equation}
f_{i}^k\left(\bm{z}_{1}^{t}, \bm{z}_{2}^{t} \cdots \bm{z}_{N}^{t}\right)= \psi_a^k (\sum_{j^t\in \mathcal{S}(i^t)} \psi_r^k([\bm{z}_i^t,\bm{z}_j^t])),    
\end{equation}
\RRRR{where $j^t$ represents the neighbor of $i$ at timestamp $t$ and the order of $\bm{z}_i^t$ and $\bm{z}_j^t$ also matters.} 
Then, we take a weighted combination of these GNN prototypes for each object, and the \RR{prototypical} interacting dynamics 
can be formulated as:
\begin{equation}\label{eq:ODE}
\frac{d\bm{z}_i^t}{dt}=\sum_{k=1}^K \bm{w}_{i}^k \psi_a^k (\sum_{j^t\in \mathcal{S}(i^t)} \psi_r^k([\bm{z}_i^t,\bm{z}_j^t])) - \bm{z}_i^t.
\end{equation}
The last term indicates natural recovery, which usually benefits semantics learning in practice. To generate the weights for each object, we merge both object-level and system-level latent variables and adopt a FFN $\rho(\cdot)$ as follows:
\begin{equation}\label{eq:weight}
    \bm{w}_i = [\bm{w}_i^1, \cdots, \bm{w}_i^K ] =\rho([\bm{u}_i,\bm{g}]),
\end{equation}
where the softmax activation is adopted to ensure $\sum_{k=1}^K \bm{w}_{i}^k =1$.

\textbf{Robustness.} In this part, we discuss the robustness of the proposed \method{}. When $K=1$, Eqn. \ref{eq:ODE} would be degraded into a single-prototype system:
\begin{equation}\label{thm1-single}
			 \frac{d\bm{z}_i^t}{dt} = \psi_{a}^1 \left( \sum_{j^t \in \mathcal{S}(i^t)} \psi_r([\bm{z}_i^t, \bm{z}_j^t]) \right) - \bm{z}_i^t,
		\end{equation}
which shares the GNN function for every node. Then, the following theorem states that our model enjoys the enhanced robustness of the proposed model to perturbation~\cite{niu2020evaluating,xu2020automatic} compared with the single-prototype system as in Eqn. \ref{eq:ODE}. Consider a perturbation $\bm{\delta}$ of small magnitude $\epsilon$, such that $\|\bm{\delta}\|=\epsilon$, applied to an given input point $\bm{Z}^0$, where $\bm{Z}^{0}=(Z_{i}^{0},\ldots,Z_{N}^{0})^{\top }$, resulting $\tilde{\bm{Z}}^0=\bm{Z}^0+\bm{\delta}$. The following theorem with the proof in Appendix \ref{proof_thm1} demonstrates that the multi-prototype system is more robust than the single-prototype system.
\begin{theorem}\label{thm1}
Assume the prototype function \( \psi_a^k \) has a bounded gradient. Moreover, each prototype function \( \psi_a^k\) and \( \psi_r^k \) are Lipschitz continuous with Lipschitz constant \( L^k_{a} \) and \( L^k_r \), and \( \psi_a\) and \( \psi_r \) are for single prototype function with Lipschitz constant \( L_a  \) and \(L_r\). For the sake of simplicity, we omit the last term $-\bm{z}_i^t$ in Eqn. \ref{eq:ODE} and Eqn. \(\ref{thm1-single}\) since it can be incorporated in the revised GNN prototypes. Denote $L^k= L^k_{a}  L^k_{r} $ and $L=L_aL_r$, if $\mathbb{E}(L^k)<\mathbb{E}(L)$, $\operatorname{Var}(L^k)<\operatorname{Var}(L)$ hold for all $k=1,\ldots,K$, our multi-prototype system described in Eqn. \(\ref{eq:ODE}\) will have smaller mean and variance bounds for the Lyapunov error function \({\|\bm{e}^t\|^2}/{2}\) compared to the single-prototype system described in Eqn. \(\ref{thm1-single}\).
\end{theorem}

\textbf{A Mixture-of-Experts Perspective.} We demonstrate that our graph ODE model can be interpreted through the lens of the mixture of experts (MoE)~\cite{du2022glam,wang2024graph,liu2023fair}. Specifically, each prototype serves as an ODE expert, while $\bm{w}_i$ acts as the gating weights that control the contribution of each expert. \RRRR{Through this, we are the first to get the graph ODE married with MoE, enhancing the expressivity to capture complex interacting dynamics as in previous works~\cite{wang2022learning,wang2020deep}.} More importantly, different from previous works that employ black-box routing functions~\cite{zhou2022mixture}, the routing function in our \method{} is derived from hierarchical contexts with representation disentanglement, which further equips our model with the generalization capability to handle potential shift in data distributions. 
In particular, given a change in the graph structure or feature distribution, the multi-prototype system Eqn. \ref{eq:ODE} can adjust the weights \( \{\bm{w}_i^{k}\} \) to accommodate this change, potentially identifying a new combination of prototypes that better fits the altered data. This flexibility is quantified by the ability to perform gradient-based updates on the weights. In contrast, Eqn. \ref{thm1-single} may fail to adapt as readily since it relies on a single function \( \psi_a \) without the benefit of re-weighting different prototypes. 

\textbf{Existence and Uniqueness.} We give a theoretical analysis about the existence and uniqueness of our proposed graph ODE to show that it is well-defined under certain conditions.

\begin{lemma}\label{lemma}
We first assume that the learnt functions $\psi_r^k:\ \mathbb{R}^{2d}\to\mathbb{R}^{d},\psi_a^k:\mathbb{R}^{d}\to\mathbb{R}^{d}$ have bounded gradients. In other words, there exists $ A, R>0$, such that the following Jacobian matrices have the bounded matrix norms:
\begin{equation}
\begin{aligned}
        J_{\psi_r^k}([\bm{x},\bm{y}]) = & \begin{pmatrix}
\frac{\partial \psi_{r,1}^k}{\partial x_1} & \cdots & \frac{\partial \psi_{r,1}^k}{\partial x_d} & \frac{\partial \psi_{r,1}^k}{\partial y_1} & \cdots & \frac{\partial \psi_{r,1}^k}{\partial y_d} \\
\vdots & \ddots & \vdots & \vdots & \ddots & \vdots \\
\frac{\partial \psi_{r,d}^k}{\partial x_1} & \cdots & \frac{\partial  \psi_{r,d}^k}{\partial x_d} & \frac{\partial  \psi_{r,d}^k}{\partial y_1} & \cdots & \frac{\partial  \psi_{r,d}^k}{\partial y_d}
\end{pmatrix}, \\  & \quad \quad \quad \|J_{\psi_r^k}([\bm{x},\bm{y}])\| \le R,
\end{aligned}
\end{equation}
\vspace{-0.5cm}
\begin{equation}
    J_{\psi_a^k}(\bm{x}) =\begin{pmatrix}
\frac{\partial \psi_{a,1}^k}{\partial x_1} & \cdots & \frac{\partial \psi_{a,1}^k}{\partial x_d}, \\
\vdots & \ddots & \vdots  \\
\frac{\partial \psi_{a,d}^k}{\partial x_1} & \cdots & \frac{\partial  \psi_{a,d}^k}{\partial x_d} 
\end{pmatrix},\quad \|J_{\psi_a^k}(\bm{x})\|\le A.
\end{equation}
Given the initial state $(t_0, \bm{z}_1^{t_0}, \cdots, \bm{z}_N^{t_0}, \bm{w}_1, \cdots, \bm{w}_N)$, we claim that there exists $\varepsilon>0$, such that the ODE system Eqn. \ref{eq:ODE} has a unique solution in the interval $[t_0-\varepsilon, t_0 + \varepsilon].$
\end{lemma}

The proof is shown in Appendix \ref{proof}. Our analysis demonstrates that based on given observations, future trajectories are predictable using our graph ODE, which is essential in dynamics modeling~\cite{chen2018neural,kong2020sde}.

\subsection{Decoder and Optimization}

Finally, we introduce a decoder to forecast future trajectories, along with an end-to-end variational
inference framework for the maximization of the likelihood.

\begin{table*}[t]
  \tabcolsep=6.5pt
  \centering
  \caption{Mean Squared Error (MSE) $\times 10^{-2}$ on physical dynamics simulations. }\label{mainresult:physical}
  \resizebox{\textwidth}{!}{
  \begin{tabular}{l|l|cc|cc|cc|cc|cc|cc}
  \toprule
  \multirow{2}{*}{Dataset} & Prediction Length & \multicolumn{2}{c|}{12 \footnotesize{(ID)}} & \multicolumn{2}{c|}{24 \footnotesize{(ID)}} & \multicolumn{2}{c|}{36 \footnotesize{(ID)}} & \multicolumn{2}{c|}{12 \footnotesize{(OOD)}} & \multicolumn{2}{c|}{24 \footnotesize{(OOD)}} & \multicolumn{2}{c}{36 \footnotesize{(OOD)}} \\ 
   & Variable & $q$ & $v$ & $q$ & $v$ & $q$ & $v$ & $q$ & $v$ & $q$ & $v$ & $q$ & $v$  \\ \midrule
  \multirow{8}{*}{\textit{Springs}}& LSTM& 0.287  & 0.920  & 0.659  & 2.659  & 1.279  & 5.729  & 0.474  & 1.157  & 0.938  & 2.656  & 1.591  & 5.223   \\
  &GRU& 0.394  & 0.597  & 0.748  & 1.856  & 1.248  & 3.446  & 0.591  & 0.708  & 1.093  & 1.945  & 1.671  & 3.423  \\
  &NODE& 0.157  & 0.564  & 0.672  & 2.414  & 1.608  & 6.232  & 0.228  & 0.791  & 0.782  & 2.530  & 1.832  & 6.009   \\
  &LG-ODE&	0.077  & 0.268  & 0.155  & 0.513  & 0.527  & 2.143  & 0.088  & 0.299  & 0.179  & 0.562  & 0.614  & 2.206  \\
  &MPNODE& 0.076  & 0.243  & 0.171  & 0.456  & 0.600  & 1.737  & 0.094  & 0.249  & 0.212  & 0.474  & 0.676  & 1.716   \\
  &SocialODE& 0.069  & 0.260  & 0.129  & 0.510  & 0.415  & 2.187  & 0.079  & 0.285  & 0.153  & 0.570  & 0.491  & 2.310   \\
  &HOPE& 0.070  & 0.176  & 0.456  & 0.957  & 2.475  & 5.409  & 0.076  & 0.221  & 0.515  & 1.317  & 2.310  & 5.996  \\
  & \cellcolor{pink}\method{} (Ours) & \cellcolor{pink}\textbf{0.035}  & \cellcolor{pink}\textbf{0.124}  & \cellcolor{pink}\textbf{0.070}  & \cellcolor{pink}\textbf{0.262}  & \cellcolor{pink}\textbf{0.296}  & \cellcolor{pink}\textbf{1.326} & \cellcolor{pink}\textbf{0.047} & \cellcolor{pink}\textbf{0.138}  & \cellcolor{pink}\textbf{0.088}  & \cellcolor{pink}\textbf{0.291}  & \cellcolor{pink}\textbf{0.309}  & \cellcolor{pink}\textbf{1.337}   \\ \midrule
  \multirow{8}{*}{\textit{Charged}}& LSTM&  0.795  & 3.029  & 2.925  & 3.734  & 6.569  & 4.331  & 1.127  & 3.027  & 3.988  & 3.640  & 8.185  & 4.221     \\
  &GRU&  0.781  & 2.997  & 2.805  & 3.640  & 5.969  & 4.147  & 1.042  & 3.028  & 3.747  & 3.636  & 7.515  & 4.101    \\
  &NODE& 0.776  & 2.770  & 3.014  & 3.441  & 6.668  & 4.043  & 1.124  & 2.844  & 3.931  & 3.563  & 8.497  & 4.737    \\
  &LG-ODE&	0.759  & 2.368  & 2.526  & 3.314  & 5.985  & 5.618  & 0.932  & 2.551  & 3.018  & 3.589  & 6.795  & 6.365   \\
  &MPNODE& 0.740  & 2.455  & 2.458  & 3.664  & 5.625  & 6.259  & 0.994  & 2.555  & 2.898  & 3.835  & 6.084  & 6.797  \\
  &SocialODE& 0.662  & 2.335  & 2.441  & 3.252  & 6.410  & 4.912  & 0.894  & 2.420  & 2.894  & 3.402  & 6.292  & 6.340    \\
  &HOPE& 0.614  & 2.316  & 3.076  & 3.381  & 8.567  & 8.458  & 0.878  & 2.475  & 3.685  & 3.430  & 10.953  & 9.120   \\ 
  &\cellcolor{pink}\method{} (Ours) & \cellcolor{pink}\textbf{0.578}  & \cellcolor{pink}\textbf{2.196}  & \cellcolor{pink}\textbf{2.037} & \cellcolor{pink}\textbf{2.648}  & \cellcolor{pink}\textbf{4.804}  & \cellcolor{pink}\textbf{3.551}  & \cellcolor{pink}\textbf{0.802}  & \cellcolor{pink}\textbf{2.135}  & \cellcolor{pink}\textbf{2.584}  & \cellcolor{pink}\textbf{2.663}  & \cellcolor{pink}\textbf{5.703}  & \cellcolor{pink}\textbf{3.703}  \\ \bottomrule
  \end{tabular}
  }
  \vspace{-0.1cm}
\end{table*}

\begin{figure*}
  \centering
  \includegraphics[width=0.93\textwidth]{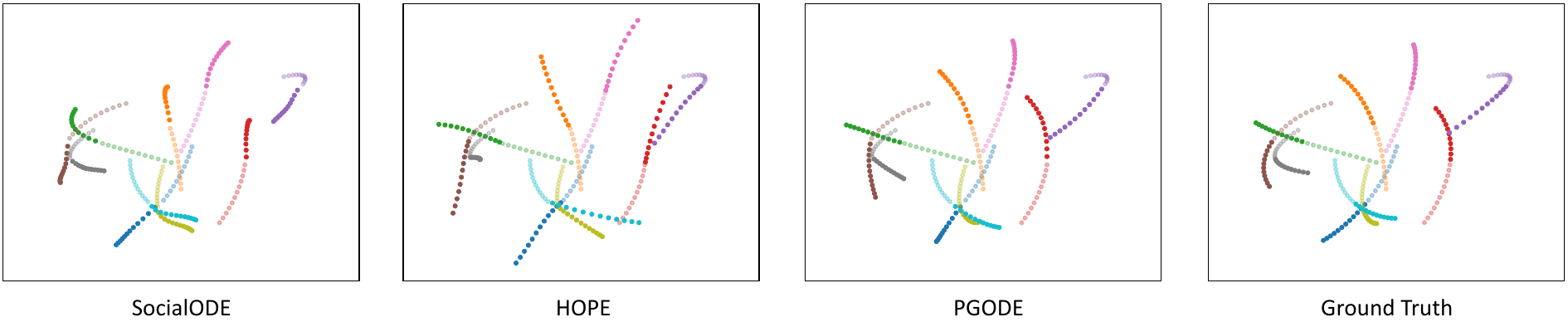}
  \vspace{-0.2cm}
  \caption{\RRRR{Visualization of different methods on \textit{Springs}. Semi-transparent paths denote observed trajectories and solid paths represent our predictions.}}
  \label{fig:vis_phy}
  \vspace{-0.5cm}
\end{figure*}

In particular, we build a connection between latent states and trajectories by calculating the likelihood for each observation $p(\bm{x}_i^t|\bm{z}_i^t)$. \RR{Following the maximum likelihood estimation of a Gaussian distribution,} here we solely produce the mean of each distribution, i.e., $\bm{\mu}_{i}^t = \phi(\bm{z}_i^t)$, where $\phi(\cdot)$ is an FFN serving as the decoder implemented. In the variational inference framework, our model optimizes the evidence lower bound (ELBO) of the likelihood, which involves the maximization of the likelihood and the minimization of the difference between the prior and posterior distributions:
\begin{equation}\label{eq:elbo}
\begin{aligned}
    \mathcal{L}_{elbo}&=\mathbb{E}_{Z^{0} \sim \prod_{i=1}^{N} q\left(\bm{z}_{i}^{0} \mid G^{1:T_{obs}}\right)}\left[\log p(\bm{X}^{T_{obs}+1: T})\right] \\ &-\operatorname{KL}\left[\prod_{i=1}^{N} q(\bm{z}_i^0| G^{1:T_{obs}}) \| p\left(\bm{Z}^{0}\right)\right],
\end{aligned}
\end{equation}
in which \RRRR{$p\left(\bm{Z}^{0}\right)=\Pi_{i=1}^N p(\bm{z}_i^0)$ and $p(\bm{z}_i^0)$ is a Normal distribution $N(\bm{0}, \bm{I})$~\cite{kingma2019introduction}.} Eqn. \ref{eq:elbo} can be re-written into the following equation by incorporating the independence of each node:
\begin{equation}\label{eq:final}
\begin{aligned}
\mathcal{L}_{elbo}&=-\sum_{i=1}^N \sum_{t= T_{obs}+1}^T\frac{\left\|\bm{x}_{i}^{t}-\bm{\mu}_{i}^{t}\right\|^{2}}{2 \sigma^{2}} \\ & -\operatorname{KL}\left[\prod_{i=1}^{N} q(\bm{z}_i^0| G^{1:T_{obs}}) \| p\left(\bm{Z}^{0}\right)\right],
\end{aligned}
\end{equation}
in which $\sigma^2$ represents the variance of the prior distribution. To summarize, the final loss objective for the optimization is written as follows:
\begin{equation}\label{eq:final_loss}
   \mathcal{L} = \mathcal{L}_{elbo} + \mathcal{L}_{sys} + \mathcal{L}_{dis},
\end{equation}
\RR{where the last two loss terms serve as a regularization mechanism using mutual information to constrain the model parameters~\cite{xu2019multi,rhodes2021local}.} 
We have summarized the whole algorithm in Appendix \ref{alg:sec}.

\section{Experiment}

We conduct experiments on both physical and molecular dynamical systems. Each sample is split into two parts including a conditional part for initializing object-level context representations and global-level context representations, and a prediction part for supervision. Their lengths are denoted as conditional length and prediction length, respectively. We compared our \method{} with several baselines, i.e., LSTM~\cite{hochreiter1997long}, GRU~\cite{cho2014learning}, NODE~\cite{chen2018neural}, LG-ODE~\cite{huang2020learning}, MPNODE~\cite{chen2022spatio}, SocialODE~\cite{wen2022social} and HOPE~\cite{luo2023hope}. The setting details are in Appendix \ref{imple_detail}.

\begin{table*}[t]
  \tabcolsep=6.6pt
  \caption{Mean Squared Error (MSE) $\times 10^{-3}$ on molecular dynamics simulations. }\label{mainresult:molecular}
  \resizebox{\textwidth}{!}{
  \begin{tabular}{l|l|ccc|ccc|ccc|ccc}
  \toprule
  \multirow{2}{*}{Dataset} &Prediction Length & \multicolumn{3}{c|}{12 \footnotesize{(ID)}} & \multicolumn{3}{c|}{24 \footnotesize{(ID)}} & \multicolumn{3}{c|}{12 \footnotesize{(OOD)}} & \multicolumn{3}{c}{24 \footnotesize{(OOD)}}  \\ 
  & Variable & $q_x$ & $q_y$ & $q_z$ & $q_x$ & $q_y$ & $q_z$ & $q_x$ & $q_y$ & $q_z$ & $q_x$ & $q_y$ & $q_z$ \\ \midrule
  \multirow{8}{*}{\textit{5AWL}}&LSTM& 4.178  & 3.396  & 3.954  & 4.358  & 4.442  & 3.980  & 4.785  & 4.178  & 4.467  & 5.152  & 5.216  & 4.548    \\
  &GRU& 4.365  & 2.865  & 2.833  & 5.295  & 3.842  & 3.996  & 5.139  & 3.662  & 3.789  & 6.002  & 4.723  & 5.358     \\
  &NODE& 3.992  & 3.291  & 2.482  & 4.674  & 4.333  & 3.254  & 4.390  & 4.135  & 2.808  & 5.734  & 5.388  & 4.036   \\
  &LG-ODE& 2.825  & 2.807  & 2.565  & 3.725  & 3.940  & 3.412  & 3.358  & 3.549  & 3.501  & 4.611  & 4.763  & 4.543   \\
  &MPNODE& 2.631  & 3.029  & 2.734  & 3.587  & 4.151  & 3.488  & 3.061  & 3.899  & 3.355  & 4.271  & 5.085  & 4.427   \\
  &SocialODE& 2.481 & 2.729 & 2.473  & 3.320  & 3.951  & 3.399  & 2.987  & 3.514  & 3.166  & 4.248  & 4.794  & 4.155       \\
  &HOPE& 2.326  & 2.572  & 2.442  & 3.495  & 3.816  & 3.413  & 2.581  & 3.528  & 2.955  & 4.548  & 5.047  & 4.007   \\
  & \cellcolor{pink}\method{} (Ours) &  \cellcolor{pink}\textbf{2.098}  &  \cellcolor{pink}\textbf{2.344}  &  \cellcolor{pink}\textbf{2.099}  &  \cellcolor{pink}\textbf{2.910}  &  \cellcolor{pink}\textbf{3.384}  &  \cellcolor{pink}\textbf{2.904}  &  \cellcolor{pink}\textbf{2.217}  &  \cellcolor{pink}\textbf{3.109}  &  \cellcolor{pink}\textbf{2.593}  &  \cellcolor{pink}\textbf{3.374}  &  \cellcolor{pink}\textbf{4.334}  &  \cellcolor{pink}\textbf{3.615}     \\  \midrule
  \multirow{8}{*}{\textit{2N5C}}&LSTM& 2.608  & 2.265  & 3.975  & 3.385  & 2.959  & 4.295  & 3.285  & 2.210  & 5.247  & 3.834  & 2.878  & 5.076     \\
  &GRU& 2.847  & 2.968  & 3.493  & 3.340  & 3.394  & 3.636  & 3.515  & 3.685  & 3.796  & 4.031  & 3.938  & 3.749    \\
  &NODE& 2.211  & 2.103  & 2.601  & 3.074  & 2.849  & 3.284  & 2.912  & 2.648  & 2.799  & 3.669  & 3.478  & 3.874   \\
  &LG-ODE& 2.176  & 1.884  & 1.928  & 2.824  & 2.413  & 2.689  & 2.647  & 2.284  & 2.326  & 3.659  & 3.120  & 3.403  \\
  &MPNODE& 1.855  & 1.923  & 2.235  & 2.836  & 2.805  & 3.416  & 2.305  & 2.552  & 2.373  & 3.244  & 3.537  & 3.220  \\
  &SocialODE& 1.965  & 1.717  & 1.817  & 2.575  & 2.286  & 2.412  & 2.348  & 2.138  & 2.169  & 3.380  & 2.990  & 3.057      \\
  &HOPE& 1.842  & 1.915  & 2.223  & 2.656  & 2.788  & 3.474  & 2.562  & 2.514  & 2.731  & 3.343  & 3.301  & 3.502  \\
  &\cellcolor{pink}\method{} (Ours) & \cellcolor{pink}\textbf{1.484}  & \cellcolor{pink}\textbf{1.424}  & \cellcolor{pink}\textbf{1.575}  & \cellcolor{pink}\textbf{1.960}  & \cellcolor{pink}\textbf{2.029}  & \cellcolor{pink}\textbf{2.119}  & \cellcolor{pink}\textbf{1.684}  & \cellcolor{pink}\textbf{1.809}  & \cellcolor{pink}\textbf{1.912}  & \cellcolor{pink}\textbf{2.464}  & \cellcolor{pink}\textbf{2.734}  & \cellcolor{pink}\textbf{2.727}   \\ \bottomrule
  \end{tabular}}
  \vspace{-0.2cm}
\end{table*}

\begin{figure*}
  \centering
  \includegraphics[width=0.95\textwidth]{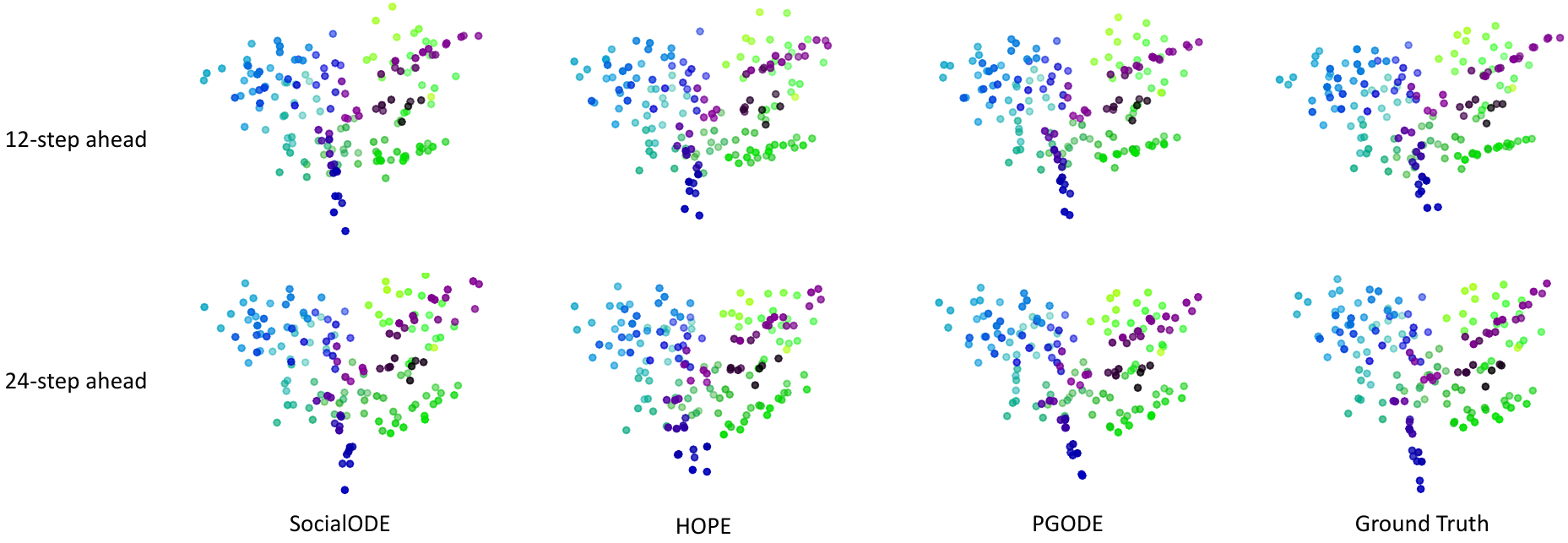}
  \vspace{-0.2cm}
  \caption{\RRRR{Visualization of prediction results of different methods on the \textit{5AWL} dataset. We can observe that our \method{} can reconstruct the ground truth accurately. }}
  \label{fig:vis_mole}
    \vspace{-0.4cm}
\end{figure*}

\subsection{Performance on Physical Dynamics Simulations}
\vspace{-0.2cm}
\textbf{Datasets.}
We employ two physics simulation datasets to evaluate our \method{}, i.e., \textit{Springs} and \textit{Charged}~\cite{kipf2018neural}.
Each sample in these two simulated datasets contains 10 particles in a 2D box that has potential collisions without exterior forces.
We aim to predict the future position information and the future velocity values of these interacting particles, i.e., $q$ and $v$.
More details of the two datasets can be found in Appendix \ref{dataset_detail}.

\textbf{Performance Comparison.} The compared results with respect to different prediction lengths are collected in Table \ref{mainresult:physical}. From the results, we have two observations. \textit{Firstly,} ODE-based methods generally outperform discrete methods, which validates that continuous methods can naturally capture system dynamics and relieve the influence of potential error accumulation. \textit{Secondly}, our proposed \method{} achieves the best performance among all the methods. In particular, the average MSE reduction of our \method{} over HOPE is 47.40\% for ID and 48.57\% for OOD settings on these two datasets. The superior performance stems from two reasons: (1) Introduction of context discovery. \method{} generates disentangled object-level and system-level embeddings, which would increase the generalization capability of the model to handle system changes, especially in OOD settings. (2) Introduction of prototype decomposition. \method{} combines a set of GNN prototypes to characterize the interacting patterns, which increases the expressivity of the model for complex dynamics. \RR{More compared results can be found in Sec. \ref{sup:perform}.}

\begin{table*}[t]
  \tabcolsep=6.2pt
  \centering
  \caption{Ablation study on \textit{Springs}, \textit{Charged} (MSE $\times 10^{-2}$) and \textit{5AWL} (MSE $\times 10^{-3}$) with a prediction length of 24. }\label{tab:ablation}
  \vspace{0.05cm}
  \resizebox{\textwidth}{!}{
  \begin{tabular}{l|cc|cc|cc|cc|ccc|ccc}
  \toprule
  Dataset & \multicolumn{2}{c|}{\textit{Springs} \footnotesize{(ID)}} & \multicolumn{2}{c|}{\textit{Springs} \footnotesize{(OOD)}} & \multicolumn{2}{c|}{\textit{Charged} \footnotesize{(ID)}} & \multicolumn{2}{c|}{\textit{Charged} \footnotesize{(OOD)}} & \multicolumn{3}{c|}{\textit{5AWL} \footnotesize{(ID)}} & \multicolumn{3}{c}{\textit{5AWL} \footnotesize{(OOD)}} \\ \midrule
  Variable & $q$ & $v$ & $q$ & $v$ & $q$ & $v$ & $q$ & $v$ & $q_x$ & $q_y$ & $q_z$ & $q_x$ & $q_y$ & $q_z$ \\ \midrule
  \method{} w/o O & 0.106 & 0.326 & 0.127 & 0.339 & 2.282 & 3.013 & 2.590 & 2.943  & 2.995 & 3.532 & 2.932 & 3.649 & 4.469 & 3.639  \\
  {\RRR{\method{} w/o $\epsilon$}} &  0.089 & 0.397 & 0.124 & 0.417  & 2.308 & 2.994 & 2.990 & 2.911  & 2.935 & 3.612 & 3.034 & 3.538 & 4.541 & 3.741   \\
  \method{} w/o F & 0.164 & 0.517 & 0.202 & 0.577 & 2.497 & 3.298 & 2.882 & 3.197   & 3.157 & 3.629 & 3.326 & 3.634 & 4.604 & 3.917   \\
  \method{} w/o D & 0.073 & 0.296 & 0.091 & 0.348 & 2.179 & 2.842 & 2.616 & 3.076    & 3.077 & 3.453 & 2.961 & 3.684 & 4.399 & 3.623    \\
  \rowcolor{pink} 
  \method{} & \textbf{0.070} & \textbf{0.262} & \textbf{0.088} & \textbf{0.291}  & \textbf{2.037} & \textbf{2.648} & \textbf{2.584} & \textbf{2.663} & \textbf{2.910} & \textbf{3.384} & \textbf{2.904} & \textbf{3.374} & \textbf{4.334} & \textbf{3.615}   \\ 
  \bottomrule
  \end{tabular}
  }
  \vspace{-0.1cm}
\end{table*}

\begin{figure*}
  \centering
  \includegraphics[width=\textwidth]{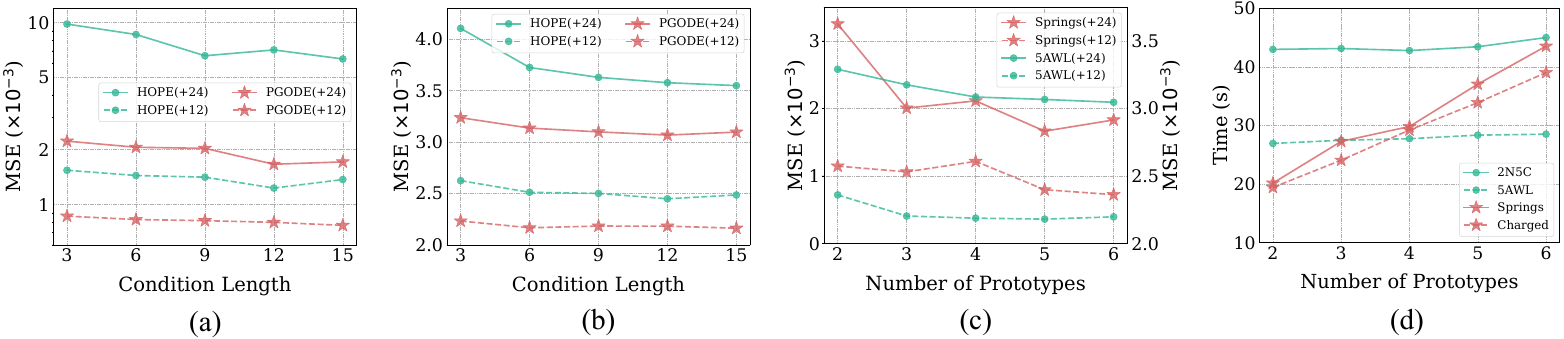}
  \vspace{-0.8cm}
  \caption{\RRRR{(a), (b) Performance with respect to varying condition lengths on \textit{Springs} and \textit{5AWL}. (c) (d) Performance and running time with respect to different numbers of prototypes.}}
  \label{fig:sensitivity}
    \vspace{-0.3cm}
\end{figure*}

\textbf{Visualization.} Figure \ref{fig:vis_phy} shows the visualization of three compared methods and the ground truth on \textit{Springs}. Here, semi-transparent paths denote the observed trajectories while solid paths denote the predicted ones. From the results, we can observe that our proposed \method{} can generate reliable trajectories close to the ground truth for all the timestamps while both baselines SocialODE and HOPE fail, which validates the superiority of the proposed \method{}.

  \vspace{-0.1cm}
\subsection{Performance on Molecular Dynamics Simulations}
  \vspace{-0.1cm}
\textbf{Datasets.} We construct two molecular dynamics datasets using two proteins, i.e., \textit{5AWL}, \textit{2N5C}, and our approach is evaluated on the two datasets.
Each sample in both datasets comprises a trajectory of molecular dynamics simulation, where the motions of each atom are governed by the Langevin dynamics equation in a specific solvent environment. {The graph is constructed by comparing pairwise distance with a threshold, which would be updated at set intervals.} The system parameters of the solvent are varied among different simulation samples.
We target at predicting the position of every atom in three coordinates, i.e., $q_x$, $q_y$ and $q_z$. \RR{More details can be found in Appendix \ref{dataset_detail}.}

\textbf{Performance Comparison.} We demonstrate the performance with respect to varying prediction lengths in Table \ref{mainresult:molecular}. Based on the results, it can be seen that our proposed \method{} can achieve the best performance on two datasets in both ID and OOD settings. Note that molecular dynamics involves hundreds of atoms with complicated interacting rules. As a consequence, the performance further demonstrates the strong expressivity of our proposed \method{} for modeling challenging underlying rules. 

\textbf{Visualization.} In addition, we provide the visualization of the two baselines and our \method{} compared with the ground truth with different prediction lengths in Figure \ref{fig:vis_mole}. We can observe that our \method{} is capable of exploring more accurate dynamical patterns compared with the ground truth. More importantly, our proposed \method{} can almost recover the position patterns when the prediction length is $24$, which validates the capability of the proposed \method{} to handle complicated scenarios.

  \vspace{-0.2cm}
\subsection{Further Analysis}
  \vspace{-0.05cm}

\textbf{Ablation Study.} To evaluate different components in \method{}, we introduce four model variants as follows: (1) \textit{\method{} w/o O}, which removes the object-level contexts and only utilizes system-level contexts for $\bm{w}_i$; (2) \textit{{\RRR{\method{} w/o $\epsilon$}}}, which removes the system-level contexts and only utilizes object-level contexts for $\bm{w}_i$; (3) \textit{\method{} w/o F}, which merely adopts one prototype for graph ODE. (4) \RRRR{\textit{\method{} w/o D}}, which remove the disentanglement loss. We compared these model variants with the full model in different settings. The results are recorded in Table \ref{tab:ablation} and more results on 2N5C can be found in Sec. \ref{sup_ablation}.
From the results, we can have several observations. \textit{Firstly}, removing either object-level or system-level contexts would obtain worse performance, which validates that both contexts are crucial to determining the interacting patterns. \textit{Secondly}, our full model achieves better performance compared with \textit{\method{} w/o F}, which validates that different prototypes can increase the representation capacity for modeling complicated dynamics. \textit{Thirdly}, in comparison to \textit{\method{} w/o D} and the full model, we can infer that representation disentanglement greatly enhances the performance under system changes.

\textbf{Parameter Sensitivity.} We first analyze the influence of different conditional lengths and prediction lengths by varying them in $\{3,6,9,12,15\}$ and $\{12,24\}$, respectively. As shown in Figure \ref{fig:sensitivity} (a) and (b), we can find that the error would decrease till saturation as the condition length rises since more historical information is provided. In addition, \method{} can always perform better than HOPE in every setting. Then, we vary the number of prototypes in $\{2,3,4,5,6\}$ in Figure \ref{fig:sensitivity} (c) and observe that more prototypes would bring in better results before saturation.

\textbf{Efficiency.} Although more prototypes tend to benefit the performance, they can also bring in high computational cost. We show the computational time with respect to different numbers of prototypes in Figure \ref{fig:sensitivity} (d) and observe that the computational complexity would increase with more prototypes. Due to the trade-off between effectiveness and efficiency, we would set the number to $5$ as default.

\section{Related Work}
\subsection{Interacting Dynamics Modeling} 

Recent years have witnessed a surge of interest in modeling interacting dynamical systems across a variety of fields including molecular biology and computational physics~\citep{lan2022dstagnn,li2022graph,bishnoi2022enhancing,sun2023unifying,yu2023learning,schaefer2021leveraging,abeyruwan2023sim2real,schlichtkrull2018modeling}. While convolutional neural networks (CNNs) have been successfully employed to learn from regular data such as grids and frames~\citep{peng2020unsteady}, emerging research is increasingly utilizing geometric graphs to represent more complex systems~\citep{wu2023learning,deng2023learning}. Graph neural networks (GNNs) have thus become increasingly prevailing for modeling these intricate dynamics. AgentFormer~\citep{yuan2021agentformer} jointly models both time and social dimensions with semantic information preserved. R-SSM~\citep{yang2020relational} models the dynamics of interacting objects using GNNs and includes auxiliary contrastive prediction tasks to enhance discriminative learning. \RRR{Equivariance is a crucial property in physical simulation to guarantee the symmetry of the physical laws and a range of previous works have been proposed~\cite{satorras2021n,wu2024equivariant}. For example, EqMotion~\cite{xu2023eqmotion} incorporates equivariant geometric feature learning for efficient multi-agent motion prediction. ESTAG~\cite{wu2024equivariant} includes the equivariant discrete Fourier transform to learn from periodic patterns.}
Despite their popularity, current methods often fall short in modeling challenging scenarios such as out-of-distribution shift and long-term dynamics~\citep{yu2021combo}. To address these limitations, our work leverages contextual knowledge to incorporate prototype decomposition into a graph ODE framework.

\subsection{Neural Ordinary Differential Equations} Motivated by the approximation of residual networks~\citep{chen2018neural}, neural ordinary differential equations (ODEs) have been introduced to model continuous-time dynamics using parameterized derivatives in hidden spaces. These neural ODEs have found widespread use in time-series forecasting due to their effectiveness~\citep{dupont2019augmented,xia2021heavy,jin2022multivariate,schirmer2022modeling}. Incorporated with the message passing mechanism, they have been integrated with GNNs, which can mitigate the issue of oversmoothing and enhance model interpretability~\citep{xhonneux2020continuous,zhang2022improving,poli2019graph}. I-GPODE~\citep{yildiz2022learning} estimates the uncertainty of trajectory predictions using the Gaussian process, which facilitates effective long-term predictions. HOPE~\citep{luo2023hope} incorporates second-order graph ODE in evolution modeling. In contrast, we not only introduce context discovery with disentanglement, which disentangles object-level and system-level embeddings with known system parameters, but also introduce prototypical graph ODE, which incorporates the object-level and system-level embeddings into prototypical graph ODE framework following the mixture-of-experts (MoE) principle.

\section{Conclusion}

In this work, we investigate a long-standing problem of modeling interacting dynamical systems and develop a novel approach named \method{}, which infers prototype decomposition from contextual discovery for a graph ODE framework. In particular, \method{} extracts disentangled object-level and system-level contexts from historical trajectories, which can enhance the capability of generalization under system changes. In addition, we present a graph ODE framework that determines a combination of multiple interacting prototypes for increased model expressivity. Extensive experiments demonstrate the superiority of our \method{} in different settings compared with various competing approaches.

\section*{Impact Statement}

This work introduces a data-driven framework for modeling interacting dynamical systems in different settings, which can be applied to facilitate research in physics and molecular biology. 
In addition, our work proposes new datasets and benchmarks on physical and molecular dynamics simulations in different settings, which can also benefit research in scientific machine learning. \RRRR{Our PGODE model can also be applied to traffic flow prediction and stock price prediction, where we can model the continuous interaction between different vehicles or stocks. In future work, we will extend \method{} to these practical problems and more scientific applications, e.g., rigid dynamics and single-cell dynamics.}

\section*{Acknowledgement}
This work was partially supported by NSF 2211557, NSF 1937599, NSF 2119643, NSF 2303037, NSF 2312501, NASA, SRC JUMP 2.0 Center, Amazon Research Awards, and Snapchat Gifts.

\bibliography{main.bib}
\bibliographystyle{icml2024}

\newpage
\appendix
\onecolumn

\section{Algorithm}\label{alg:sec}

We summarize the learning algorithm of our \method{} in Algorithm \ref{alg}. 
\begin{algorithm}[!h]
\caption{Training Algorithm of \method{}}
\label{alg}
\begin{flushleft}
\textbf{Input:} The observations $G^{1:T} = \{G^1, \cdots, G^T\}$.   \\
\textbf{Output}: The parameters in the model.  \\
\end{flushleft}
\begin{algorithmic}[1] 
\STATE Initialize model parameters;
\WHILE{not convergence}
\FOR{each training sequence}
\STATE Partition the sequence into two parts;
\STATE Construct the temporal graph with Eqn. \ref{eq:temporal_graph};
\STATE Generate object-level contexts using Eqn. \ref{eq:local}; 
\STATE Generate system-level contexts with summarization; 
\STATE Solve our \RR{prototypical graph ODE} in Eqn. \ref{eq:ODE};
\STATE Output the trajectories using the decoder;
\STATE Compute the final objective, i.e., Eqn. \ref{eq:final_loss};
\STATE \RR{Update $\tau'$ in our \method{} with gradient ascent;}
\STATE \RR{Update other parameters in our \method{} using gradient descent;}
\ENDFOR
\ENDWHILE
\end{algorithmic}
\end{algorithm}

\section{Proof of Theorem \ref{thm1}}\label{proof_thm1}

\textbf{Theorem 3.1.} \textit{Assume the prototype function \( \psi_a^k \) has a bounded gradient. Moreover, each prototype function \( \psi_a^k\) and \( \psi_r^k \) are Lipschitz continuous with Lipschitz constant \( L^k_{a} \) and \( L^k_r \), and \( \psi_a\) and \( \psi_r \) are for single prototype function with Lipschitz constant \( L_a  \) and \(L_r\). For the sake of simplicity, we omit the last term $-\bm{z}_i^t$ in Eqn. \ref{eq:ODE} and Eqn. \(\ref{thm1-single}\) since it can be incorporated in the revised GNN prototypes. Denote $L^k= L^k_{a}  L^k_{r} $ and $L=L_aL_r$, if $\mathbb{E}(L^k)<\mathbb{E}(L)$, $\operatorname{Var}(L^k)<\operatorname{Var}(L)$ hold for all $k=1,\ldots,K$, our multi-prototype system described in Eqn. \(\ref{eq:ODE}\) will have smaller mean and variance bounds for the Lyapunov error function \({\|\bm{e}^t\|^2}/{2}\) compared to the single-prototype system described in Eqn. \(\ref{thm1-single}\).}

\begin{proof}
To show that the multi-prototype system Eqn. \ref{eq:ODE} is better in terms of robustness compared to the single-prototype system Eqn. \ref{thm1-single}, we consider a perturbation \(\bm\delta\) of small magnitude \(\epsilon\), with \(\|\bm\delta\| = \epsilon\), applied to an input point \(\bm{Z}^0\). For the multi-prototype system, we assume the perturbed solution of the ODE at time \(t\) is \(\bm{\tilde{Z}}^t\). After omitting the last term in Eqn. \ref{eq:ODE}, the difference between the perturbed and unperturbed states is:
\[
\left[\frac{d(\bm{\tilde{Z}}^t - \bm{Z}^t)}{dt}\right]_{i} = \sum_{k=1}^{K} w_i^k \left( \psi_a^k \left( \sum_{j \in \mathcal{S}(i^t)} \psi_r^k ([\tilde{Z}_i^t, \tilde{Z}_j^t]) \right) - \psi_a^k \left( \sum_{j \in \mathcal{S}(i^t)} \psi_r^k ([Z_i^t, Z_j^t]) \right) \right) ,
\]
where $[\bm{a}]_{i}$ denotes the $i$th element of the vector $\bm{a}$.
For the single-prototype system Eqn. \ref{thm1-single},  the difference between the perturbed and unperturbed states is:
\[
\left[\frac{d(\bm{\tilde{Z}}^t - \bm{Z}^t)}{dt}\right]_{i} = \psi_a^1 \left( \sum_{j \in \mathcal{S}(i^t)} \psi_r ([\tilde{Z}_i^t, \tilde{Z}_j^t]) \right) - \psi_a^1 \left( \sum_{j \in \mathcal{S}(i^t)} \psi_r ([Z_i^t, Z_j^t]) \right) .
\]
Consider the error \(\bm{e}^{t} = \bm{\tilde{Z}}^t - \bm{Z}^t\) and analyze its growth over time. For the multi-prototype system, the combined effect of multiple prototypes with weights \(w_i^k\) tends to average out the perturbation, potentially leading to a slower growth of \(\|\bm{e}^t\|\). This can be quantified by:
\begin{equation}\label{eq:err-p-m}
\left[\frac{d\bm{e}^t}{dt}\right]_i = \sum_{k=1}^{K} w_i^k \left( \psi_a^k \left( \sum_{j \in \mathcal{S}(i^t)} \psi_r^k ([\tilde{Z}_i^t, \tilde{Z}_j^t]) \right) - \psi_a^k \left( \sum_{j \in \mathcal{S}(i^t)} \psi_r^k ([Z_i^t, Z_j^t]) \right) \right) .
\end{equation}
For the single-prototype system, the error propagation is:
\begin{equation}\label{eq:err-p-s}
\left[\frac{d\bm{e}^t}{dt}\right]_i = \psi_a^1 \left( \sum_{j \in \mathcal{S}(i^t)} \psi_r ([\tilde{Z}_i^t, \tilde{Z}_j^t]) \right) - \psi_a^1 \left( \sum_{j \in \mathcal{S}(i^t)} \psi_r ([Z_i^t, Z_j^t]) \right) .
\end{equation}
We use Lyapunov functions to quantify the stability. A Lyapunov function \(V(x)\) is a scalar function that maps the state of the system to a non-negative real number. In this case, we define a Lyapunov function as:
\[
V(\bm{e}^t) = \frac{1}{2} \|\bm{e}^t\|^2.
\]
The time derivative of \(V\) is then given by:
\[
\frac{dV}{dt} = \frac{d}{dt} \left( \frac{1}{2} \|\bm{e}^t\|^2 \right) = (\bm{e}^t)^{\top} \cdot \frac{d\bm{e}^t}{dt}.
\]
Plugging in Eqn. \ref{eq:err-p-m} and Eqn. \ref{eq:err-p-s} into the above derivative, we get:
\begin{equation}\label{eq:Lyp-m}
\frac{dV}{dt} =\sum_{i=1}^{N} [\bm{e}^t]_{i} \cdot \left[\frac{d\bm{e}^t}{dt}\right]_{i} =\sum_{i=1}^{N} [\bm{e}^t]_{i} \cdot \left( \sum_{k=1}^{K} w_i^k \left( \psi_a^k \left( \sum_{j \in \mathcal{S}(i^t)} \psi_r^k ([\tilde{Z}_i^t, \tilde{Z}_j^t]) \right) - \psi_a^k \left( \sum_{j \in \mathcal{S}(i^t)} \psi_r^k ([Z_i^t, Z_j^t]) \right) \right) \right)
\end{equation}
and
\begin{equation}\label{eq:Lyp-s}
\frac{dV}{dt} = \sum_{i=1}^{N}[\bm{e}^t]_{i} \cdot \left[\frac{d\bm{e}^t}{dt}\right]_{i} =\sum_{i=1}^{N} [\bm{e}^t]_{i}\cdot \left( \psi_a^1 \left( \sum_{j \in \mathcal{S}(i^t)} \psi_r ([\tilde{Z}_i^t, \tilde{Z}_j^t]) \right) - \psi_a^1 \left( \sum_{j \in \mathcal{S}(i^t)} \psi_r ([Z_i^t, Z_j^t]) \right) \right).
\end{equation}
Assume \(\psi_a\) and \(\psi_r\) are Lipschitz continuous with Lipschitz constants \(L_a\) and \(L_r\), respectively. Eqn. \ref{eq:Lyp-m} and Eqn. \ref{eq:Lyp-s} can be approximated by:

\begin{align*}
	&\left|\frac{dV}{dt}\right| \leq \sum_{i=1}^{N} |[\bm{e}^t]_{i}|\cdot\left( \sum_{k=1}^{K} w_i^k L_a^k L_r^k\|\bm{e}^t \| \right)  \leq \|\bm{e}^t\|^2\sqrt{\sum_{i=1}^{N}\left( \sum_{k=1}^{K} w_i^k L_a^k L_r^k\right)^2}
\end{align*}
and
\[
\left|\frac{dV}{dt}\right|^2\leq\left( \sum_{i=1}^{N} |[\bm{e}^t]_{i}| L_a L^k\|\bm{e}^t\|\right)^2     \leq\|\bm{e}^t \|^2\sqrt{ N } L_a L_r
\]

Under the assumptions in Theorem \ref{thm1}, we have \(L^{k}=L_a^k L_r^k\), \(L_a L_r = L\), and:

\begin{equation}
\mathbb{E}\left(\sum_{k=1}^{K} w_i^k L^{k}\right) < \mathbb{E}\left(\sum_{k=1}^{K} w_i^k L\right) = \mathbb{E}(L),
\end{equation}
and
\begin{equation}
\operatorname{Var}\left(\sum_{k=1}^{K} w_i^k L^{k}\right) = \sum_{k=1}^{K} (w_i^k)^2 \operatorname{Var}(L^k) < \sum_{k=1}^{K} (w_i^k)^2 \operatorname{Var}(L) \leq \operatorname{Var}(L),
\end{equation}
where the last equality holds iff there exists \(k^{\ast}\) such that \(w_i^{k^\ast}=1\) and \(w_i^k=0\) for all \(k \neq k^{\ast}\). This shows that the multi-prototype system has a distributed effect that reduces the impact of perturbations, leading to slower growth in \(\frac{dV}{dt}\) compared to the single-prototype system, which has a concentrated effect. This means that the multi-prototype system is more robust.

\end{proof}

\section{Proof of Lemma \ref{lemma} }\label{proof}

\textbf{Lemma 3.2.}
We first assume that the learnt functions $\psi_r^k:\ \mathbb{R}^{2d}\to\mathbb{R}^{d},\psi_a^k:\mathbb{R}^{d}\to\mathbb{R}^{d}$ have bounded gradients. In other words, there exists $ A, R>0$, such that the following Jacobian matrices have bounded matrix norm:
\begin{equation}
    J_{\psi_r^k}([\bm{x},\bm{y}]) = \begin{pmatrix}
\frac{\partial \psi_{r,1}^k}{\partial x_1} & \cdots & \frac{\partial \psi_{r,1}^k}{\partial x_d} & \frac{\partial \psi_{r,1}^k}{\partial y_1} & \cdots & \frac{\partial \psi_{r,1}^k}{\partial y_d} \\
\vdots & \ddots & \vdots & \vdots & \ddots & \vdots \\
\frac{\partial \psi_{r,d}^k}{\partial x_1} & \cdots & \frac{\partial  \psi_{r,d}^k}{\partial x_d} & \frac{\partial  \psi_{r,d}^k}{\partial y_1} & \cdots & \frac{\partial  \psi_{r,d}^k}{\partial y_d}
\end{pmatrix},\quad  \|J_{\psi_r^k}([\bm{x},\bm{y}])\|\le R,
\end{equation}

\begin{equation}
    J_{\psi_a^k}(\bm{x}) =\begin{pmatrix}
\frac{\partial \psi_{a,1}^k}{\partial x_1} & \cdots & \frac{\partial \psi_{a,1}^k}{\partial x_d}, \\
\vdots & \ddots & \vdots  \\
\frac{\partial \psi_{a,d}^k}{\partial x_1} & \cdots & \frac{\partial  \psi_{a,d}^k}{\partial x_d} 
\end{pmatrix},\quad \|J_{\psi_a^k}(\bm{x})\|\le A.
\end{equation}

Then, given the initial state $(t_0, \bm{z}_1^{t_0}, \cdots, \bm{z}_N^{t_0}, \bm{w}_1, \cdots, \bm{w}_N)$, we claim that there exists $\varepsilon>0$, such that the ODE system Eqn. \ref{eq:ODE} has a unique solution in the interval $[t_0-\varepsilon, t_0 + \varepsilon].$

We first introduce the Picard–Lindel$\ddot{o}$f Theorem as below. 
\begin{theorem}\label{eq:theorm}(Picard–Lindel$\ddot{o}$f Theorem~\cite{coddington1956theory})
Let $D \subseteq \mathbb{R} \times \mathbb{R}^n$ be a closed rectangle with $\left(t_0, y_0\right) \in D$. Let $f: D \rightarrow \mathbb{R}^n$ be a function which is continuous in $t$ and Lipschitz continuous in $y$. In this case, there exists some $\varepsilon>0$ such that the initial value problem:
\begin{equation}
    y^{\prime}(t)=f(t, y(t)), \quad y\left(t_0\right)=y_0 .
\end{equation}
has a unique solution $y(t)$ on the interval $\left[t_0-\varepsilon, t_0+\varepsilon\right].$
\end{theorem}

Then, we prove the following lemma. 

\begin{lemma}\label{lemma:new}
Suppose we have a series of L-Lipschitz continuous functions $\{f_i:\mathbb{R}^m \to \mathbb{R}^n\}_{i=1}^N$, and then their linear combination is also L-Lipschitz continuous, i.e., $\forall \{a_1,\cdots a_N\}\in [0,1]^N$, satisfying $\sum_{i=1}^N a_i = 1$, we have $\sum_{i=1}^N a_if_i$ is also L-Lipschitz continuous.
    
\end{lemma}

\begin{proof}
$\forall \bm{x},\bm{y}\in \mathbb{R}^m$, we have:
\begin{align}
    \|\sum_{i=1}^N a_if_i(\bm{x})- \sum_{i=1}^N a_if_i(\bm{y})\|&\le\sum_{i=1}^N a_i\|f_i(\bm{x})-f_i(\bm{y})\|\\
    &\le \sum_{i=1}^N a_i L\|\bm{x}-\bm{y}\| \\&= L\|\bm{x}-\bm{y}\|.
\end{align}
\end{proof}

Next, we show the proof of Lemma \ref{lemma}.

\begin{proof}

First, we can rewrite the ODE system Eqn. \ref{eq:ODE} as: 
\begin{equation}
\frac{d\bm{Z}^t}{dt} = \sum_{k=1}^K \bm{W}^k f^k(\bm{Z}^t) -\bm{Z}^t ,
\end{equation}
where $\bm{W}^k\in \mathbb{R}^{Nd \times Nd}$ is a diagonal matrix.
It is evident that the right hand side is continuous with respect to $t$ since it does not depend on $t$ directly. 

Then, for any continuous function $f: \mathbb{R}^n\to \mathbb{R}^m$, with the Mean Value Theorem, we have $\forall \bm{x},\bm{y}\in \mathbb{R}^n, \|f(\bm{x}) - f(\bm{y})\| = \|J_f(\bm{p})\|*\|\bm{x}-\bm{y}\|,$ where $\bm{p}$ is a point in the segment connecting $\bm{x}$ and $\bm{y}$.
Also, denote
\begin{equation}
J_{\psi_r^k,\bm{x}}([\bm{x},\bm{y}]) =  \begin{pmatrix}
\frac{\partial \psi_{r,1}^k}{\partial x_1} & \cdots & \frac{\partial \psi_{r,1}^k}{\partial x_d} \\
\vdots & \ddots & \vdots \\
\frac{\partial \psi_{r,d}^k}{\partial x_1} & \cdots & \frac{\partial  \psi_{r,d}^k}{\partial x_d}
\end{pmatrix}, 
  J_{\psi_r^k,\bm{y}}([\bm{x},\bm{y}]) =  \begin{pmatrix}
\frac{\partial \psi_{r,1}^k}{\partial y_1} & \cdots & \frac{\partial \psi_{r,1}^k}{\partial y_d} \\
\vdots & \ddots & \vdots \\
\frac{\partial \psi_{r,d}^k}{\partial y_1} & \cdots & \frac{\partial  \psi_{r,d}^k}{\partial y_d}
\end{pmatrix}.
\end{equation}
By assumption, we have
\begin{equation}
     \|J_{\psi_r^k,\bm{x}}([\bm{x},\bm{y}])\|, \quad \|J_{\psi_r^k,\bm{y}}([\bm{x},\bm{y}])\| \le \|J_{\psi_r^k}([\bm{x},\bm{y}])\| \le R.
\end{equation}

Now, denote $A(i,j)\in\mathbb{R}^{2d\times dN}$. For the indices from (1, idN+1) to (d, (i+1)dN), and from (d+1, jdN+1) to (2d, (j+1)dN), the matrix value is 1; the others are 0. By introducing $A(i,j)$, for all  $\bm{X} = \begin{pmatrix}
    \bm{x}_1\\
    \vdots\\
    \bm{x}_N
\end{pmatrix}, \bm{Y} = 
\begin{pmatrix}
    \bm{y}_1\\
    \vdots\\
    \bm{y}_N
\end{pmatrix} \in\mathbb{R}^{dN},$ we have:
\begin{align}
\|\psi_r^k(A(i,j)\bm{X})-\psi_r^k(A(i,j)\bm{Y})\|&\le \|\psi_r^k([\bm{x}_i, \bm{x}_j])-\psi_r^k([\bm{y}_i, \bm{x}_j])\|+\|\psi_r^k([\bm{y}_i, \bm{x}_j])-\psi_r^k([\bm{y}_i, \bm{y}_j])\|\\
 &= \|J_{\psi_r^k, \bm{x}}([\bm{p}_i,\bm{x}_j])\|*\|\bm{x}_i-\bm{y}_i\|+\|J_{\psi_r^k,\bm{y}}([\bm{y}_i,\bm{p}_j])\|*\|\bm{x}_j-\bm{y}_j\| \quad(MVT)\\
&\le R\|\bm{x}_i-\bm{y}_i\|+R\|\bm{x}_j-\bm{y}_j\|\\
&\le R\|\bm{X}- \bm{Y}\|,
\end{align}
where $\bm{p}_i$ is a point in the segment connecting $\bm{x}_i$ and $\bm{y}_i$, and a similar definition is for $\bm{p}_j$. Note that we have $\psi_r^k$ is R-Lipschitz continuous. Therefore, by Lemma \ref{lemma:new}, the following linear combination is also R-Lipschitz continuous:
\begin{equation}
l^k(\bm{Z}^t) = \sum_{j^t\in \mathcal{S}(i^t)} \psi_r^k([A(i^t,j^t)\bm{Z}^t]).    
\end{equation}
Thus, for all $ \bm{X},\bm{Y}\in \mathbb{R}^{dN}$, we have:
\begin{align}
\|f^k(\bm{X})-f^k(\bm{Y})\| &= \|\psi_a^k(l^k(\bm{X}))-\psi_a^k(l^k(\bm{Y}))\|\\
&\le A\|l^k(\bm{X})-l^k(\bm{Y})\| \\& \le ARN \|\bm{X}-\bm{Y}\|.    
\end{align}
Again, we have each $f^k$ is ARN-Lipschitz continuous, so their linear combination $\sum_{k=1}^K \bm{W}^k f^k$ will also be Lipschitz continuous. 
Finally, we have
\begin{align}
\|[\sum_{k=1}^K \bm{W}^k f^k(\bm{X}) - \bm{X}]-[\sum_{k=1}^K \bm{W}^k f^k(\bm{Y}) -\bm{Y}]\| &\le \|\sum_{k=1}^K \bm{W}^k f^k(\bm{X}) -\sum_{k=1}^K \bm{W}^k f^k(\bm{Y})\|\\&+\|\bm{X}-\bm{Y}\|\\
&\le (ARNK+1)\|\bm{X}-\bm{Y}\|.    
\end{align}
Thus, the right hand side will be (ARNK+1)-Lipschitz continuous. According to the Theorem \ref{eq:theorm}, we prove the uniqueness of the solution to Eqn. \ref{eq:ODE}.
\end{proof}

\section{More Related Work}

\subsection{Graph Neural Networks}
Graph Neural Networks (GNNs)~\citep{kipf2016semi,xu2019powerful,velivckovic2018graph,feng2023towards,ju2024comprehensive,lienen2023generative,steeven2024space} have shown remarkable efficacy in handling a range of graph-based machine learning tasks such as node classification~\citep{yang2021hgat} and graph classification~\citep{liu2022graph}. Typically, they adopt the message passing mechanism, where each node aggregates messages from its adjacent nodes for updated node representations. Recently, researchers have started to focus more on realistic graphs that do not obey the homophily assumption and developed several GNN approaches to tackle heterophily~\citep{zhu2021graph,li2022finding,zhu2020beyond}. These approaches typically leverage new graph structures~\citep{zhu2020beyond,suresh2021breaking} and modify the message passing procedures~\citep{chien2021adaptive,yan2022two} to mitigate the influence of potential heterophily. 
In our \method{}, we focus on interacting dynamics systems instead. In particular, due to the local heterophily, different objects should have different interacting patterns, and therefore we infer object-level contexts from historical data.

\RRR{\section{Limitation}

One limitation of our PGODE is that it does not consider the symmetry of physics, which is an important property in physical simulations~\cite{satorras2021n,xu2023eqmotion,wu2024equivariant}. In future works, we will incorporate the symmetry of physics to further enhance the expressivity of our method, which builds high-quality equivariant graph ODE models for dynamical system modeling.}

\section{Detail of Baselines}
The proposed method is compared with these competing baselines as follows:
\begin{itemize}[leftmargin=*]
    \item LSTM~\citep{hochreiter1997long} has been broadly utilized for sequence prediction tasks. Compared with classic RNNs, LSTM incorporates three critical gates, i.e., the forget gate, the input gate, and the output gate, which can effectively understand and retain important long-term dependencies within the data sequences.
    \item GRU~\citep{cho2014learning} 
    is another popular RNN architecture, which employs the gating mechanism to control the information flow during propagation. GRU has an improved computational efficiency compared LSTM. 
    \item NODE~\citep{chen2018neural} is the first method to introduce a continuous neural network based on the residual connection. It has been shown effective in time-series forecasting. 
    \item LG-ODE~\citep{huang2020learning} incorporates graph neural networks with neural ODE, which can capture continuous interacting dynamics in irregularly-sampled partial observations.
    \item MP-NODE~\citep{gupta2022learning} integrate the message passing mechanism into neural ODEs, which can capture sub-system relationships during the evolution of homogeneous systems. 
    \item SocialODE~\citep{wen2022social} simulates the evolution of agent states and their interactions using a neural ODE architecture, which shows remarkable performance in multi-agent trajectory forecasting.
    \item HOPE~\citep{luo2023hope} is a recently proposed graph ODE method, which leverages a twin encoder to learn hidden representations. These representations are fed into a high-order graph ODE to learn long-term correlations from complicated dynamical systems. 
    \RRR{\item EGNN~\cite{satorras2021n} is a graph neural network architecture, which promises the equivalence to E(3) transformations. It shows superior performance for learning from physical simulations. 
    \item EqMotion~\cite{xu2023eqmotion} is an efficient model, which includes both an equivariant geometric feature learning module and an invariant pattern feature for comprehensive motion prediction. 
    }
\end{itemize}

\section{Dataset Details}\label{dataset_detail}

We use four simulation datasets to evaluate our proposed \method{}, including physical and molecular dynamic systems.
We will introduce the details of these four datasets in this part.

\begin{itemize}[leftmargin=*]
\item \textit{Springs} \& \textit{Charged}. The two physical dynamic simulation datasets \textit{Springs} and \textit{Charged} are commonly used in the field of machine learning for simulating physical systems. 
The \textit{Springs} dataset simulates a system of interconnected springs governed by Hooke's law.
Each spring has inherent properties such as elasticity coefficients and initial positions, representing a dynamic mechanical system.
Each sample in the \textit{Springs} dataset contains 10 interacting springs with information about the current state, i.e., velocity and acceleration, and additional properties, i.e., mass and damping coefficients.
Similar to the \textit{Springs} dataset, \textit{Charged} is another popular physical dynamic simulation dataset that simulates electromagnetic phenomena.
The objects in \textit{Charged} are replaced by the electronics.
We use the box size $\alpha$, the initial velocity $\beta$, the interaction strength $\gamma$, and spring\/charged probability $\delta$ as the system parameters in the experiments.
It is noteworthy that the objects attract or repel with equal probability in the \textit{Charged} system but unequal probability in the spring system. \RR{Both systems have a given graph indicating fixed interactions from real springs or electric charge effects. }

\item \textit{5AWL} \& \textit{2N5C}. To evaluate our approach on modeling molecular
dynamic systems, we construct two datasets from two proteins, \textit{5AWL} and \textit{2N5C}, which can be accessed from the RCSB\footnote{https://www.rcsb.org}.
First, we repair missing residues, non-standard residues, missing atoms, and hydrogen atoms in the selected protein. 
Additionally, we adjust the size of the periodic boundary box to ensure that it is sufficiently large, thus avoiding truncation effects and abnormal behavior of the simulation system during the data simulation process.
Then, we perform simulations on the irregular molecular motions within the protein using Langevin Dynamics~\citep{garcia1998langevin} under the NPT (isothermal-isobaric ensemble) conditions, with parameters sampled from the specified range, and we extract a frame every 0.2 $ps$ to record the protein structure, which constitutes the dataset used for supervised learning. 
In the two constructed datasets, we use the temperature $t$, pressure value $p$, and frictional coefficient $\mu$ as the dynamic system parameters.
Langevin Dynamics is a mathematical model used to simulate the flow dynamics of molecular systems~\citep{bussi2007accurate}. 
It can simplify complex systems by replacing some degrees of freedom of the molecules with stochastic differential equations.
For a dynamic system containing $N$ particles of mass $m$, with coordinates given by $X = X(t)$, the Langevin equation of it can be formulated as follows:
\begin{equation}
    m\frac{d^{2}X}{dt^{2}} = -\Delta{U}(X)-\mu\frac{dX}{dt}+\sqrt{2\mu{k_b}T}R(t),
\end{equation}
where $\mu$ represents the frictional coefficient, $\Delta{U}(X)$ is the interaction potential between particles, $\Delta$ is the gradient operator, $T$ is the temperature, $k_b$ is Boltzmann constant and $R(t)$ denotes the delta-correlated stationary Gaussian process.
\end{itemize}

\section{Implementation Details}\label{imple_detail}

\begin{table*}[t]
\centering
\caption{Datasets and distributions of system parameters. For the OOD test set, there is at least one of the system parameters outside the range utilized for training. $\alpha$: box size, $\beta$: initial velocity norm,$\gamma$: interaction strength, $\delta$: spring/charged probability. $t$: temperature, $p$: pressure, $\mu$: frictional coefficient. }
\vspace{0.05cm}
\begin{tabular}{lcccl}
    \toprule
  & \textit{Springs}            & \textit{Charged}                & \textit{5AWL}/\textit{2N5C} &  \\
\cmidrule(r){2-4}
Parameters       & $\alpha, \beta, \gamma, \delta$              & $\alpha, \beta, \gamma, \delta$           &    $t, p, \mu$       &  \\
Train/Val/Test      & $\begin{array}{c}
A = \{\alpha \in [4.9, 5.1] \}\\
B = \{\beta\in[0.49, 0.51]\} \\ 
C = \{\gamma\in[0.09, 0.11]\}\\
D = \{\delta\in [0.49, 0.51]\} \\
\Omega_{\text{train}} = (A \times B \times C \times D) \\
\end{array}$ & 
$\begin{array}{c}
A = \{\alpha \in [4.9, 5.1] \}\\
B = \{\beta\in[0.49, 0.51]\} \\ 
C = \{\gamma\in[0.9, 1.1]\}\\
D = \{\delta\in [0.49, 0.51]\} \\
\Omega_{\text{train}} = (A \times B \times C \times D) \\
\end{array}$ & $\begin{array}{c}
  T = \{t \in [290, 310] \}\\
  P = \{p\in[0.9, 1.1]\} \\ 
  M = \{\mu\in[0.9, 1.1]\}\\
  \Omega_{\text{train}} = (T \times P \times M) \\
  \end{array}$
  \\
  \\
OOD Test Set    &   $\begin{array}{c} 
A = \{\alpha \in [4.8, 5.2] \}\\
B = \{\beta\in[0.48, 0.52]\} \\ 
C = \{\gamma\in[0.08, 0.12]\}\\
D = \{\delta\in [0.48, 0.52]\} \\
\Omega_{\text{OOD}} = \\ (A \times B \times C \times D) \setminus \Omega_{\text{train}} \\
\textcolor{white}{.} \\
\end{array}$  &  
$\begin{array}{c} 
A = \{\alpha \in [4.8, 5.2] \}\\
B = \{\beta\in[0.48, 0.52]\} \\ 
C = \{\gamma\in[0.8, 1.2]\}\\
D = \{\delta\in [0.48, 0.52]\} \\
\Omega_{\text{OOD}} = \\ (A \times B \times C \times D) \setminus \Omega_{\text{train}} \\
\textcolor{white}{.} \\
\end{array}$ & 
$\begin{array}{c} 
  T = \{t \in [280, 320] \}\\
  P = \{p\in[0.8, 1.2]\} \\ 
  M = \{\mu\in[0.8, 1.2]\}\\
  \Omega_{\text{OOD}} = \\ (T \times P \times M) \setminus \Omega_{\text{train}} \\
  \textcolor{white}{.} \\
  \end{array}$ \\
\midrule
Number of samples \\ 
Train/Val/Test  & \multicolumn{2}{c}{1000/200/200} & 200/50/50\\
OOD Test Set& \multicolumn{2}{c}{200} & 50  \\
\bottomrule
\end{tabular}
\label{tab:datasets}
\end{table*}

\begin{table*}[t]
  \tabcolsep=6.5pt
    \vspace{-0.2cm}
  \caption{Performance comparison with NRI, AgentFormer, and I-GPODE on physical dynamics simulations (MSE $\times 10^{-2}$). NRI, AgentFormer, and I-GPODE are out of memory on molecular dynamics simulations.}\label{supresult:physical}
  \resizebox{\textwidth}{!}{
  \begin{tabular}{l|l|cc|cc|cc|cc|cc|cc}
  \toprule
  \multirow{2}{*}{Dataset} &Prediction Length & \multicolumn{2}{c|}{12 \footnotesize{(ID)}} & \multicolumn{2}{c|}{24 \footnotesize{(ID)}} & \multicolumn{2}{c|}{36 \footnotesize{(ID)}} & \multicolumn{2}{c|}{12 \footnotesize{(OOD)}} & \multicolumn{2}{c|}{24 \footnotesize{(OOD)}} & \multicolumn{2}{c}{36 \footnotesize{(OOD)}} \\ 
  &Variable & $q$ & $v$ & $q$ & $v$ & $q$ & $v$ & $q$ & $v$ & $q$ & $v$ & $q$ & $v$  \\ \midrule
  \multirow{4}{*}{\textit{Springs}}& NRI & 0.103 & 0.425 & 0.210 & 0.681 & 0.693 & 2.263 & 0.119 & 0.472 & 0.246 & 0.770 & 0.807 & 2.406  \\
  &AgentFormer & 0.115 & 0.163 & 0.202 & 0.517 & 1.656 & 1.691 & 0.157 & 0.195 & 0.243 & 0.505 & 1.875 & 1.913  \\
  &I-GPODE & 0.159 & 0.479 & 0.746 & 3.002 & 1.701 & 7.433 & 0.173 & 0.498 & 0.796 & 3.193 & 1.818 & 7.322  \\
  &\cellcolor{pink}\method{} (Ours) & \cellcolor{pink}\textbf{0.035}  & \cellcolor{pink}\textbf{0.124}  & \cellcolor{pink}\textbf{0.070}  & \cellcolor{pink}\textbf{0.262}  & \cellcolor{pink}\textbf{0.296}  & \cellcolor{pink}\textbf{1.326} & \cellcolor{pink}\textbf{0.047} & \cellcolor{pink}\textbf{0.138}  & \cellcolor{pink}\textbf{0.088}  & \cellcolor{pink}\textbf{0.291}  & \cellcolor{pink}\textbf{0.309}  & \cellcolor{pink}\textbf{1.337}   \\ \midrule
  \multirow{4}{*}{\textit{Charged}}&NRI & 0.901 & 2.702 & 3.225 & 3.346 & 7.770 & 4.543 & 1.303 & 2.726 & 3.678 & 3.548 & 8.055 & 4.752  \\
  &AgentFormer & 1.076 & 2.476 & 3.631 & 3.044 & 7.513 & 3.944 & 1.384 & 2.514 & 4.224 & 3.199 & 8.985 & 4.002  \\
  &I-GPODE & 1.044 & 2.818 & 3.407 &	3.751 & 7.292 & 4.570 & 1.322 & 2.715 & 3.805 & 3.521 & 8.011 & 4.056  \\ 
  &\cellcolor{pink}\method{} (Ours) & \cellcolor{pink}\textbf{0.578}  & \cellcolor{pink}\textbf{2.196}  & \cellcolor{pink}\textbf{2.037} & \cellcolor{pink}\textbf{2.648}  & \cellcolor{pink}\textbf{4.804}  & \cellcolor{pink}\textbf{3.551}  & \cellcolor{pink}\textbf{0.802}  & \cellcolor{pink}\textbf{2.135}  & \cellcolor{pink}\textbf{2.584}  & \cellcolor{pink}\textbf{2.663}  & \cellcolor{pink}\textbf{5.703}  & \cellcolor{pink}\textbf{3.703}  \\ \bottomrule
  \end{tabular}}
  \vspace{-0.1cm}
\end{table*}

\begin{table*}[h]
  \tabcolsep=6pt
  \caption{Mean Squared Error (MSE) $\times 10^{-2}$ on \textit{Springs}. }\label{table:5}
  \resizebox{\textwidth}{!}{
  \begin{tabular}{l|l|cccc|cccc|cccc}
  \toprule
  \multirow{2}{*}{Distribution} &Prediction Length & \multicolumn{4}{c|}{12} & \multicolumn{4}{c|}{24} & \multicolumn{4}{c}{36} \\
  &Variable & $q_x$ & $q_y$ & $v_x$ & $v_y$ & $q_x$ & $q_y$ & $v_x$ & $v_y$ & $q_x$ & $q_y$ & $v_x$ & $v_y$  \\ \midrule
  \multirow{8}{*}{ID}&LSTM& 0.324 & 0.250 & 0.909 & 0.931 & 0.679 & 0.638 & 2.695 & 2.623 & 1.253 & 1.304 & 5.023 & 6.434   \\
  &GRU& 0.496 & 0.291 & 0.565 & 0.628 & 0.873 & 0.623 & 1.711 & 2.001 & 1.368 & 1.128 & 2.980 & 3.912  \\
  &NODE& 0.165 & 0.148 & 0.649 & 0.479 & 0.722 & 0.621 & 2.534 & 2.293 & 1.683 & 1.534 & 6.323 & 6.142    \\
  &LG-ODE&	0.077 & 0.077 & 0.264 & 0.272 & 0.174 & 0.135 & 0.449 & 0.576 & 0.613 & 0.441 & 1.757 & 2.528   \\
  &MPNODE& 0.080 & 0.072 & 0.222 & 0.263 & 0.237 & 0.105 & 0.407 & 0.506 & 0.866 & 0.335 & 1.469 & 2.006   \\
  &SocialODE& 0.069 & 0.068 & 0.205 & 0.315 & 0.138 & 0.120 & 0.391 & 0.630 & 0.429 & 0.400 & 1.751 & 2.624    \\
  &HOPE& 0.087 & 0.053 & 0.152 & 0.200 & 0.571 & 0.342 & 0.707 & 1.206 & 2.775 & 2.175 & 4.412 & 6.405  \\
  &\cellcolor{pink}\method{} (Ours) & \cellcolor{pink}\textbf{0.033} & \cellcolor{pink}\textbf{0.037} & \cellcolor{pink}\textbf{0.122} & \cellcolor{pink}\textbf{0.127} & \cellcolor{pink}\textbf{0.074} & \cellcolor{pink}\textbf{0.066} & \cellcolor{pink}\textbf{0.239} & \cellcolor{pink}\textbf{0.286} & \cellcolor{pink}\textbf{0.318} & \cellcolor{pink}\textbf{0.273} & \cellcolor{pink}\textbf{1.186} & \cellcolor{pink}\textbf{1.466}    \\ \midrule
  \multirow{8}{*}{OOD}&LSTM&  0.499 & 0.449 & 1.086 & 1.227 & 1.019 & 0.857 & 2.847 & 2.466 & 1.768 & 1.415 & 5.154 & 5.293      \\
  &GRU&  0.714 & 0.469 & 0.713 & 0.703 & 1.280 & 0.905 & 1.795 & 2.096 & 1.844 & 1.497 & 2.852 & 3.994    \\
  &NODE& 0.246 & 0.209 & 0.997 & 0.585 & 0.876 & 0.687 & 2.790 & 2.269 & 2.002 & 1.663 & 6.349 & 5.670    \\
  &LG-ODE&	0.093 & 0.083 & 0.272 & 0.327 & 0.185 & 0.172 & 0.463 & 0.661 & 0.684 & 0.545 & 1.767 & 2.645   \\
  &MPNODE& 0.107 & 0.081 & 0.230 & 0.268 & 0.299 & 0.126 & 0.420 & 0.528 & 0.967 & 0.386 & 1.464 & 1.969   \\
  &SocialODE& 0.082 & 0.076 & 0.221 & 0.350 & 0.151 & 0.156 & 0.414 & 0.726 & 0.488 & 0.495 & 1.793 & 2.826    \\
  &HOPE& 0.094 & 0.058 & 0.178 & 0.264 & 0.506 & 0.523 & 1.031 & 1.603 & 2.369 & 2.251 & 3.701 & 8.291   \\
  &\cellcolor{pink}\method{} (Ours) & \cellcolor{pink}\textbf{0.046} & \cellcolor{pink}\textbf{0.048} & \cellcolor{pink}\textbf{0.133} & \cellcolor{pink}\textbf{0.144} & \cellcolor{pink}\textbf{0.094} & \cellcolor{pink}\textbf{0.081} & \cellcolor{pink}\textbf{0.286} & \cellcolor{pink}\textbf{0.297} & \cellcolor{pink}\textbf{0.336} & \cellcolor{pink}\textbf{0.281} & \cellcolor{pink}\textbf{1.360} & \cellcolor{pink}\textbf{1.313}  \\ \bottomrule
  \end{tabular}}
  \vspace{-0.1cm}
\end{table*}

In our experiments, we employ a rigorous data split strategy to ensure the accuracy of our results.
Specifically, we split the whole datasets into four different parts, including the normal three sets, i.e., training, validating and in-distribution (ID) test sets and an out-of-distribution (OOD) test set. 
For the physical dynamic datasets, we generate 1200 samples for training and validating, 200 samples for ID testing and 200 samples for OOD testing.
For the molecular dynamic datasets, we construct 200 samples for training, 50 samples for validating, 50 samples for ID testing and 50 samples for testing in OOD settings.

Each sample in the datasets has a group of distinct system parameters as shown in Table~\ref{tab:datasets}. For training, validation and ID test samples, we randomly sample system parameters in the space of $\Omega_{train}$. For OOD samples, the system parameters come from $\Omega_{OOD}$ randomly, which indicates the distribution shift compared with the training domain. 
In our experiments, we set the conditional length to 12, and we used three different prediction lengths, i.e., 12, 24, and 36.

We leverage PyTorch~\citep{paszke2017automatic} and torchdiffeq package~\citep{kidger2021hey} to implement all the compared approaches and our \method{}. All these experiments in this work are performed on a single NVIDIA A40 GPU. The fourth-order Runge-Kutta method from torchdiffeq is adopted as the ODE solver. We employ a set of one-layer GNN prototypes with a hidden dimension of 128 for graph ODE. The number of prototypes is set to $5$ as default. For optimization, we utilize an Adam optimizer \citep{DBLP:journals/corr/KingmaB14} with an initial learning rate of $0.0005$. The batch size is set to 256 for the physical dynamic simulation datasets and 64 for the molecular dynamic simulation datasets. \RRRR{In some real-world applications, we could face region-level contexts, which influence the dynamics of a group of objects. The potential solution is to learn the embedding of region-level contexts and then incorporate them into prototypical graph ODE, i.e., replacing $\bm{\omega}_i=\rho([\bm{u}_i,\bm{g}_i])$ with $\bm{\omega}_i=\rho([\bm{u}_i,\bm{r}_i,\bm{g}_i])$ where $\bm{r}_i$ denotes the region-level embedding.}

\section{More Experiment Results}

\subsection{Performance Comparison}\label{sup:perform}

To begin, we compare with our \method{} with more baselines, i.e., AgentFormer~\citep{yuan2021agentformer}, NRI~\citep{kipf2018neural} and I-GPODE~\citep{yildiz2022learning} in our performance comparison. 
We also compared our \method{} with two equivalence-based methods, i.e., EGNN~\cite{satorras2021n} and EqMotion~\cite{xu2023eqmotion}. 
The results of these comparisons are presented in Table \ref{supresult:physical} and our method outperforms the compared methods. In addition, we show the performance of the compared methods in two different coordinates of positions and velocities, i.e., $q_x$, $q_y$, $v_x$ and $v_y$. The compared results on \textit{Springs} and \textit{Charged} are shown in Table \ref{table:5} and Table \ref{table:6}, respectively. 
\RRR{The compared results of our methods and equivalence-based methods are shown in Table \ref{supresult:egnn}. From the results, we can observe the superiority of the proposed \method{} in capturing complicated interacting patterns under both ID and OOD settings. In particular, compared with EGNN, our method can model continuous and complicated dynamics with better performance.}

\RRRR{Besides, we triple the number of agents in physical dynamics simulations. The compared results are shown in Table \ref{supresult:obj_triple}. We can observe that our proposed PGODE surpasses the performance of baseline models, highlighting the superiority of the proposed method. The compared performance on COVID-19~\cite{luo2023hope} can be seen in Table \ref{supresult:covid}. From the results, we can further validate the superiority of the proposed PGODE in real-world datasets.}

\begin{table*}[t]
  \tabcolsep=6pt
  \caption{Mean Squared Error (MSE) $\times 10^{-2}$ on  \textit{Charged}. }\label{table:6}
  \resizebox{\textwidth}{!}{
  \begin{tabular}{l|l|cccc|cccc|cccc}
  \toprule
  \multirow{2}{*}{Distribution} &Prediction Length & \multicolumn{4}{c|}{12} & \multicolumn{4}{c|}{24} & \multicolumn{4}{c}{36} \\ 
  &Variable & $q_x$ & $q_y$ & $v_x$ & $v_y$ & $q_x$ & $q_y$ & $v_x$ & $v_y$ & $q_x$ & $q_y$ & $v_x$ & $v_y$  \\ \midrule
  \multirow{8}{*}{ID}&LSTM& 0.743 & 0.846 & 2.913 & 3.145 & 2.797 & 3.052 & 3.605 & 3.863 & 6.477 & 6.660 & 4.240 & 4.423   \\
  &GRU& 0.764 & 0.799 & 2.931 & 3.063 & 2.709 & 2.901 & 3.572 & 3.709 & 5.657 & 6.281 & 4.068 & 4.227  \\
  &NODE& 0.743 & 0.808 & 2.764 & 2.777 & 2.913 & 3.114 & 3.432 & 3.451 & 6.468 & 6.868 & 3.997 & 4.089    \\
  &LG-ODE&	0.736 & 0.783 & 2.322 & 2.414 & 2.320 & 2.731 & 3.361 & 3.268 & 5.188 & 6.782 & 6.194 & 5.043   \\
  &MPNODE& 0.720 & 0.759 & 2.414 & 2.496 & 2.379 & 2.536 & 3.589 & 3.738 & 5.636 & 5.614 & 5.472 & 7.046   \\
  &SocialODE& 0.630 & 0.695 & 2.311 & 2.358 & 2.252 & 2.631 & 3.509 & 2.995 & 5.743 & 7.076 & 5.701 & 4.122   \\
  &HOPE& 0.593 & 0.635 & 2.295 & 2.337 & 3.214 & 2.938 & 3.279 & 3.482 & 9.289 & 7.845 & 8.406 & 8.511  \\
  &\cellcolor{pink}\method{} (Ours) & \cellcolor{pink}\textbf{0.555} & \cellcolor{pink}\textbf{0.600} & \cellcolor{pink}\textbf{2.164} & \cellcolor{pink}\textbf{2.228} & \cellcolor{pink}\textbf{1.940} & \cellcolor{pink}\textbf{2.134} & \cellcolor{pink}\textbf{2.624} & \cellcolor{pink}\textbf{2.673} & \cellcolor{pink}\textbf{4.449} & \cellcolor{pink}\textbf{5.159} & \cellcolor{pink}\textbf{3.778} & \cellcolor{pink}\textbf{3.324}    \\ \midrule
  \multirow{8}{*}{OOD}&LSTM&  1.130 & 1.123 & 3.062 & 2.992 & 4.026 & 3.950 & 3.768 & 3.512 & 7.934 & 8.435 & 4.517 & 3.925       \\
  &GRU&  1.072 & 1.012 & 3.108 & 2.948 & 3.893 & 3.602 & 3.844 & 3.428 & 6.970 & 8.061 & 4.485 & 3.718     \\
  &NODE& 1.185 & 1.062 & 2.956 & 2.732 & 4.057 & 3.804 & 3.645 & 3.480 & 8.622 & 8.372 & 5.097 & 4.376     \\
  &LG-ODE&	0.999 & 0.866 & 2.581 & 2.521 & 2.797 & 3.239 & 4.200 & 2.978 & 5.996 & 7.593 & 8.422 & 4.309    \\
  &MPNODE& 1.092 & 0.897 & 2.487 & 2.623 & 2.967 & 2.828 & 3.670 & 4.001 & 6.051 & 6.118 & 6.029 & 7.566    \\
  &SocialODE& 0.865 & 0.924 & 2.481 & 2.359 & 2.610 & 3.177 & 3.968 & 2.836 & \textbf{5.482} & 7.102 & 8.530 & 4.150    \\
  &HOPE& 0.839 & 0.918 & 2.466 & 2.484 & 3.586 & 3.783 & 3.417 & 3.442 & 11.254 & 10.652 & 10.133 & 8.107   \\
  &\cellcolor{pink}\method{} (Ours) & \cellcolor{pink}\textbf{0.739} & \cellcolor{pink}\textbf{0.865} & \cellcolor{pink}\textbf{2.159} & \cellcolor{pink}\textbf{2.110} & \cellcolor{pink}\textbf{2.524} & \cellcolor{pink}\textbf{2.643} & \cellcolor{pink}\textbf{2.704} & \cellcolor{pink}\textbf{2.623} & \cellcolor{pink}5.748 & \cellcolor{pink}\textbf{5.659} & \cellcolor{pink}\textbf{4.017} & \cellcolor{pink}\textbf{3.389}  \\ \bottomrule
  \end{tabular}}
  \vspace{-0.1cm}
\end{table*}

\begin{table*}[h]
  \tabcolsep=6.5pt
    \vspace{-0.2cm}
    \centering
  \caption{\RRR{Performance comparison with EGNN, EqMotion, and \method{} on physical dynamics simulations (MSE $\times 10^{-2}$).}}\label{supresult:egnn}
  \begin{tabular}{l|cc|cc|cc|cc}
  \toprule
  Dataset & \multicolumn{4}{c|}{\textit{Springs}} & \multicolumn{4}{c}{\textit{Charged}} \\ 
  Prediction Length & \multicolumn{2}{c|}{12 \footnotesize{(ID)}} & \multicolumn{2}{c|}{12 \footnotesize{(OOD)}} & \multicolumn{2}{c|}{12 \footnotesize{(ID)}} & \multicolumn{2}{c}{12 \footnotesize{(OOD)}} \\ 
  Variable &$q_x$&$q_y$&$q_x$&$q_y$&$q_x$&$q_y$&$q_x$&$q_y$  \\ \midrule
  EGNN	&0.140&	0.147&	0.150&	0.149&	2.092&	2.227&	2.139&	2.244\\
EqMotion&	0.077&	0.080&	0.084&	0.080&	0.807&	0.893&	0.867&	0.936\\
\cellcolor{pink}\method{} (Ours)&\cellcolor{pink}\textbf{0.033}&\cellcolor{pink}\textbf{0.037}&\cellcolor{pink}\textbf{0.046}&\cellcolor{pink}\textbf{0.048}&\cellcolor{pink}\textbf{0.555}&\cellcolor{pink}\textbf{0.600}&\cellcolor{pink}\textbf{0.739}&\cellcolor{pink}\textbf{0.865}\\
  \bottomrule
  \end{tabular}
  \vspace{-0.1cm}
\end{table*}

\begin{table*}[!h]
  \tabcolsep=6.5pt
  \centering
  \caption{\RRR{Performance comparison on \textit{Springs} (MSE $\times 10^{-2}$) with triple number of objects.} }\label{supresult:obj_triple}
  \resizebox{\textwidth}{!}{
  \begin{tabular}{l|cc|cc|cc|cc|cc|cc}
  \toprule
    Prediction Length & \multicolumn{2}{c|}{12 \footnotesize{(ID)}} & \multicolumn{2}{c|}{24 \footnotesize{(ID)}} & \multicolumn{2}{c|}{36 \footnotesize{(ID)}} & \multicolumn{2}{c|}{12 \footnotesize{(OOD)}} & \multicolumn{2}{c|}{24 \footnotesize{(OOD)}} & \multicolumn{2}{c}{36 \footnotesize{(OOD)}} \\ 
    Variable & $q$ & $v$ & $q$ & $v$ & $q$ & $v$ & $q$ & $v$ & $q$ & $v$ & $q$ & $v$  \\ \midrule
  
    SocialODE      & 0.152   & 0.364& 0.521   & 0.950& 2.438   & 3.785& 0.275    & 0.584& 0.687    & 1.044& 2.544    & 3.981\\
    HOPE           & 0.070   & 0.195& 0.734   & 1.892& 3.571   & 5.766& 0.241    & 0.592& 0.893    & 1.840& 3.972    & 5.841\\
   \cellcolor{pink}\method{} (Ours) & \cellcolor{pink}\textbf{0.059}& \cellcolor{pink}\textbf{0.126}& \cellcolor{pink}\textbf{0.179   }& \cellcolor{pink}\textbf{0.471}& \cellcolor{pink}\textbf{1.150   }& \cellcolor{pink}\textbf{2.041}& \cellcolor{pink}\textbf{0.224    }& \cellcolor{pink}\textbf{0.415}& \cellcolor{pink}\textbf{0.464    }& \cellcolor{pink}\textbf{0.886}& \cellcolor{pink}\textbf{1.686    }& \cellcolor{pink}\textbf{2.145}\\
  \bottomrule
  \end{tabular}
  }
  \vspace{-0.1cm}
\end{table*}

\begin{table*}[!h]
  \tabcolsep=6.5pt
    \vspace{-0.2cm}
    \centering
  \caption{\RRRR{Performance comparison on COVID-19.}}\label{supresult:covid}
  \begin{tabular}{l|cc|cc|cc}
  \toprule
  \multirow{2}{*}{Method} & \multicolumn{2}{c|}{1-week-ahead} & \multicolumn{2}{c|}{2-week-ahead} & \multicolumn{2}{c}{3-week-ahead}\\ 
  &MAE&RMSE&MAE&RMSE&MAE&RMSE  \\ \midrule
  MPNODE&152.7&237.5&272.0&549.4&248.7&385.8\\
HOPE  &85.64 & 146.0 &180.9 & 275.2  &243.1 & 373.3 \\
\cellcolor{pink}\method{} (Ours)&\cellcolor{pink}\textbf{82.99} &\cellcolor{pink}\textbf{129.2} &\cellcolor{pink}\textbf{165.2} &\cellcolor{pink}\textbf{250.6} &\cellcolor{pink}\textbf{220.6}&\cellcolor{pink}\textbf{325.4} \\
  \bottomrule
  \end{tabular}
  \vspace{-0.1cm}
\end{table*}

\begin{table*}[!h]
  \tabcolsep=7pt
  \centering
  \caption{Ablation study on \textit{2N5C} (MSE $\times 10^{-3}$) with a prediction length of 24. }\label{tab:ablation_sup}
  \begin{tabular}{l|ccc|ccc}
  \toprule
  Dataset & \multicolumn{3}{c|}{\textit{2N5C} \footnotesize{(ID)}} & \multicolumn{3}{c}{\textit{2N5C} \footnotesize{(OOD)}} \\ \midrule
  Variable  & $q_x$ & $q_y$ & $q_z$ & $q_x$ & $q_y$ & $q_z$ \\ \midrule
  \method{} w/o O    & 2.076 & 2.130 & 2.215 & 2.582 & 2.800 & 2.833  \\
  {\RRR{\method{} w/o $\epsilon$}}   & 2.040 & 2.046 & 2.227 & 2.559 & 2.791 & 2.854    \\
  \method{} w/o F   & 2.424 & 2.208 & 2.465 & 2.970 & 2.868 & 3.118    \\
  \method{} w/o D    & 2.119 & 2.083 & 2.171 & 2.785 & 2.759 & 2.829   \\
  \rowcolor{pink} 
  \method{}   & \textbf{1.960} & \textbf{2.029} & \textbf{2.119} & \textbf{2.464} & \textbf{2.734} & \textbf{2.727}   \\ 
  \bottomrule
  \end{tabular}
  \vspace{-0.3cm}
\end{table*}

\begin{table*}[!t]
  \tabcolsep=6.8pt
  \centering
  \caption{Further ablation study on \textit{Springs} (MSE $\times 10^{-2}$) and \textit{5AWL} (MSE $\times 10^{-3}$) with a prediction length of 24.}\label{tab:ablation_further}
  \vspace{0.05cm}
  \begin{tabular}{l|cc|cc|ccc|ccc}
  \toprule
  Dataset & \multicolumn{2}{c|}{\textit{Springs} \footnotesize{(ID)}} & \multicolumn{2}{c|}{\textit{Springs} \footnotesize{(OOD)}} & \multicolumn{3}{c|}{\textit{5AWL} \footnotesize{(ID)}} & \multicolumn{3}{c}{\textit{5AWL} \footnotesize{(OOD)}} \\ \midrule
  Variable & $q$ & $v$ & $q$ & $v$ & $q_x$ & $q_y$ & $q_z$ & $q_x$ & $q_y$ & $q_z$ \\ \midrule
  \method{} w. Single &  0.208 & 0.434 & 0.248 & 0.481   & 3.010 & 3.741 & 3.143 & 3.523 & 4.691 & 3.839    \\
  \method{} w. MLP & 0.152 & 0.454 & 0.179 & 0.514   & 2.997 & 3.638 & 3.240 & 3.605 & 4.492 & 3.908  \\
  \rowcolor{pink} 
  \method{} & \textbf{0.070} & \textbf{0.262} & \textbf{0.088} & \textbf{0.291}   & \textbf{2.910} & \textbf{3.384} & \textbf{2.904} & \textbf{3.374} & \textbf{4.334} & \textbf{3.615}   \\ 
  \bottomrule
  \end{tabular}
  \vspace{-0.1cm}
\end{table*}

\begin{table*}[!h]
  \tabcolsep=6.5pt
  \centering
  \caption{\RRRR{Performance comparison with a model variant, i.e., PGODE-S on \textit{Springs} (MSE $\times 10^{-2}$).} }\label{supresult:pgode-s}
  \resizebox{\textwidth}{!}{
  \begin{tabular}{l|cc|cc|cc|cc|cc|cc}
  \toprule
    Prediction Length & \multicolumn{2}{c|}{12 \footnotesize{(ID)}} & \multicolumn{2}{c|}{24 \footnotesize{(ID)}} & \multicolumn{2}{c|}{36 \footnotesize{(ID)}} & \multicolumn{2}{c|}{12 \footnotesize{(OOD)}} & \multicolumn{2}{c|}{24 \footnotesize{(OOD)}} & \multicolumn{2}{c}{36 \footnotesize{(OOD)}} \\ 
    Variable & $q$ & $v$ & $q$ & $v$ & $q$ & $v$ & $q$ & $v$ & $q$ & $v$ & $q$ & $v$  \\ \midrule
    SocialODE&0.069&	0.260&	0.129&	0.510&	0.415&	2.187&	0.079&	0.285&	0.153&	0.570&	0.491&	2.310\\
    HOPE&0.070&	0.176&	0.456&	0.957&	2.475&	5.409&	0.076&	0.221&	0.515&	1.317&	2.310&	5.996\\
    PGODE&0.035&	0.124&	0.070&	0.262&	0.296&	1.326&	0.047&	0.138&	0.088&	0.291&	0.309&	1.337\\
    PGODE-S &0.038&	0.129&	0.095&	0.298&	0.406&	1.416&	0.051&	0.148&	0.114&	0.319&	0.423&	1.411\\
  \bottomrule
  \end{tabular}
  }
  \vspace{-0.1cm}
\end{table*}

\begin{figure*}[!h]
  \centering
  \includegraphics[width=0.95\textwidth]{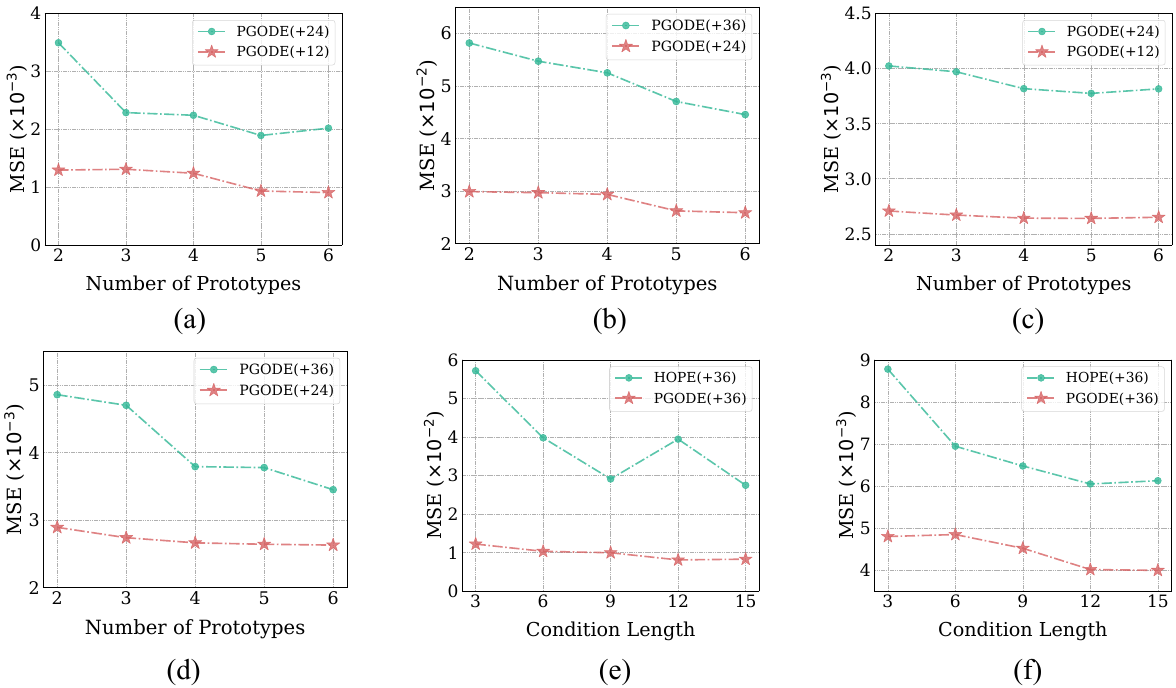}
  \vspace{-0.3cm}
  \caption{{(a),(b),(c),(d) Performance on the OOD test set of \textit{Springs}, \textit{Charged}, \textit{5AWL}, and \textit{2N5C} with respect to four different numbers of prototypes. (e),(f) Performance with respect to different condition lengths on the ID test set of \textit{Springs} and \textit{5AWL}.}}
  \label{fig:sensitivity_sup}
  \vspace{-0.5cm}
\end{figure*}

\vspace{-0.2cm}
\subsection{Ablation Study}\label{sup_ablation}
\vspace{-0.2cm}

We show more ablation studies on \textit{Charged} and \textit{2N5C} to make our analysis complete. In particular, the compared performance of different model variants are shown in Table \ref{tab:ablation_sup}. From the results, we can observe that our full model can outperform all the model variance in all cases, which validates the effectiveness of each component in our \method{} again.
\RR{In addition, we introduce two model variants: (1) \method{} w. MLP, which combines a GNN with an MLP to learn the individualized dynamics; (2) \method{} w. Single, which takes the node representation and the global representation as input with a single message passing function. The compared performance of different model variants is shown in Table \ref{tab:ablation_further}. From the results, we can observe that our full model can outperform all the model variance in all cases. Compared with these variants, our prototype decomposition can involve different GNN bases, which model diverse evolving patterns to jointly determine the individualized dynamics. This strategy can enhance the model expressivity, allowing for more accurate representation learning of hierarchical structures from a mixture-of-experts perspective.}

\RRRR{To enhance the practical utility of our method in real-world settings, we propose a model variant PGODE-S, which utilizes the top-k GNN prototypes instead of all the prototypes to enhance the efficiency. The compared performance can be found in Table \ref{supresult:pgode-s}. We can observe that although PGODE-S includes fewer parameters, its performance is still competitive, which enhances the practical utility of our model.
}

\begin{table*}[!h]
  \tabcolsep=6.6pt
  \caption{Performance comparison with different types of GNNs on \textit{5AWL} (MSE $\times 10^{-3}$). }\label{mainresult:GNN_type}
  \resizebox{\textwidth}{!}{
  \begin{tabular}{l|ccc|ccc|ccc|ccc}
  \toprule
  Prediction Length & \multicolumn{3}{c|}{12 \footnotesize{(ID)}} & \multicolumn{3}{c|}{24 \footnotesize{(ID)}} & \multicolumn{3}{c|}{12 \footnotesize{(OOD)}} & \multicolumn{3}{c}{24 \footnotesize{(OOD)}}  \\ 
   Variable & $q_x$ & $q_y$ & $q_z$ & $q_x$ & $q_y$ & $q_z$ & $q_x$ & $q_y$ & $q_z$ & $q_x$ & $q_y$ & $q_z$ \\ \midrule
   PGODE w. GIN        & 2.126& 2.426& 2.216& 2.968& 3.496& 3.003& 2.327& 3.173& 2.614&  3.573&  4.395&  3.618\\
   PGODE w. GraphSAGE  & 2.136& 2.399& 2.154& 2.935& 3.488& 3.014& 2.294& 3.158& \textbf{2.591}&  3.536&  4.442&  3.620\\
   \cellcolor{pink}\method{} w. GCN (Ours) &  \cellcolor{pink}\textbf{2.098}  &  \cellcolor{pink}\textbf{2.344}  &  \cellcolor{pink}\textbf{2.099}  &  \cellcolor{pink}\textbf{2.910}  &  \cellcolor{pink}\textbf{3.384}  &  \cellcolor{pink}\textbf{2.904}  &  \cellcolor{pink}\textbf{2.217}  &  \cellcolor{pink}\textbf{3.109}  &  \cellcolor{pink}2.593  &  \cellcolor{pink}\textbf{3.374}  &  \cellcolor{pink}\textbf{4.334}  &  \cellcolor{pink}\textbf{3.615}     \\  \bottomrule
  \end{tabular}}
  \vspace{-0.2cm}
\end{table*}

\RRR{\subsection{Performance with Different Backbone Architectures}
In this part, we explore different types of GNNs, e.g., GCN~\cite{kipf2016semi}, GIN~\cite{xu2019powerful} and GraphSAGE~\cite{hamilton2017inductive}. The results are shown in Table \ref{mainresult:GNN_type}. From the results, we can find that GCN is slightly better than other types, which helps us make the choice. Therefore, we use GCN as the default backbone for 5AWL.}

\subsection{Performance with Different Number of Prototypes}

Figure \ref{fig:sensitivity_sup} (a) (b) (c) and (d) record the performance with respect to different numbers of prototypes on different datasets. From the results, we can find that more prototypes would bring in better results before saturation. \RRRR{In practice, we can use the maximum number of prototypes in our device initially and then consider reducing it if it will not influence the performance seriously.}

\subsection{Performance with Different Condition Lengths}

We analyze the influence of different conditional lengths by varying them in $\{3,6,9,12,15\}$, respectively. As shown in Figure \ref{fig:sensitivity_sup} (e) and (f), we can observe that our \method{} can always outperform the latest baseline HOPE, which validates the superiority of the proposed \method{}. 

\RR{\subsection{Efficiency Comparison}

We have conducted a comparison of computation cost. The results are shown in Table \ref{tab:cost} and we can observe that our method has a competitive computation cost. In particular, the performance of HOPE is much worse than ours (the increasement of ours is over 47\% compared with HOPE), while our computational burden only increases a little. Moreover, both the performance and efficiency of I-GPODE are worse than ours. }

\begin{table*}[!h]
  \tabcolsep=6pt
  \centering
  \caption{Comparison of training cost per epoch (s). }\label{tab:cost}
  \vspace{0.05cm}
  \begin{tabular}{lccccccccc}
  \toprule
  Method & LSTM & GRU & NODE & LG-ODE & MPNODE & SocialODE & I-GPODE & HOPE & \method{} (Ours) \\ \midrule
  Springs &  1.53 & 1.04 & 2.21 & 17.39 & 23.33 & 21.02 & 267.08 & 23.86 & 37.03    \\
  Charged & 1.33 & 1.02 & 2.06 & 16.59 & 22.26 & 19.93 & 250.23 & 20.43 & 33.88    \\
  \bottomrule
  \end{tabular}
  \vspace{-0.1cm}
\end{table*}

\subsection{Visualization}

\begin{figure*}[!h]
  \centering
  \includegraphics[width=0.88\textwidth]{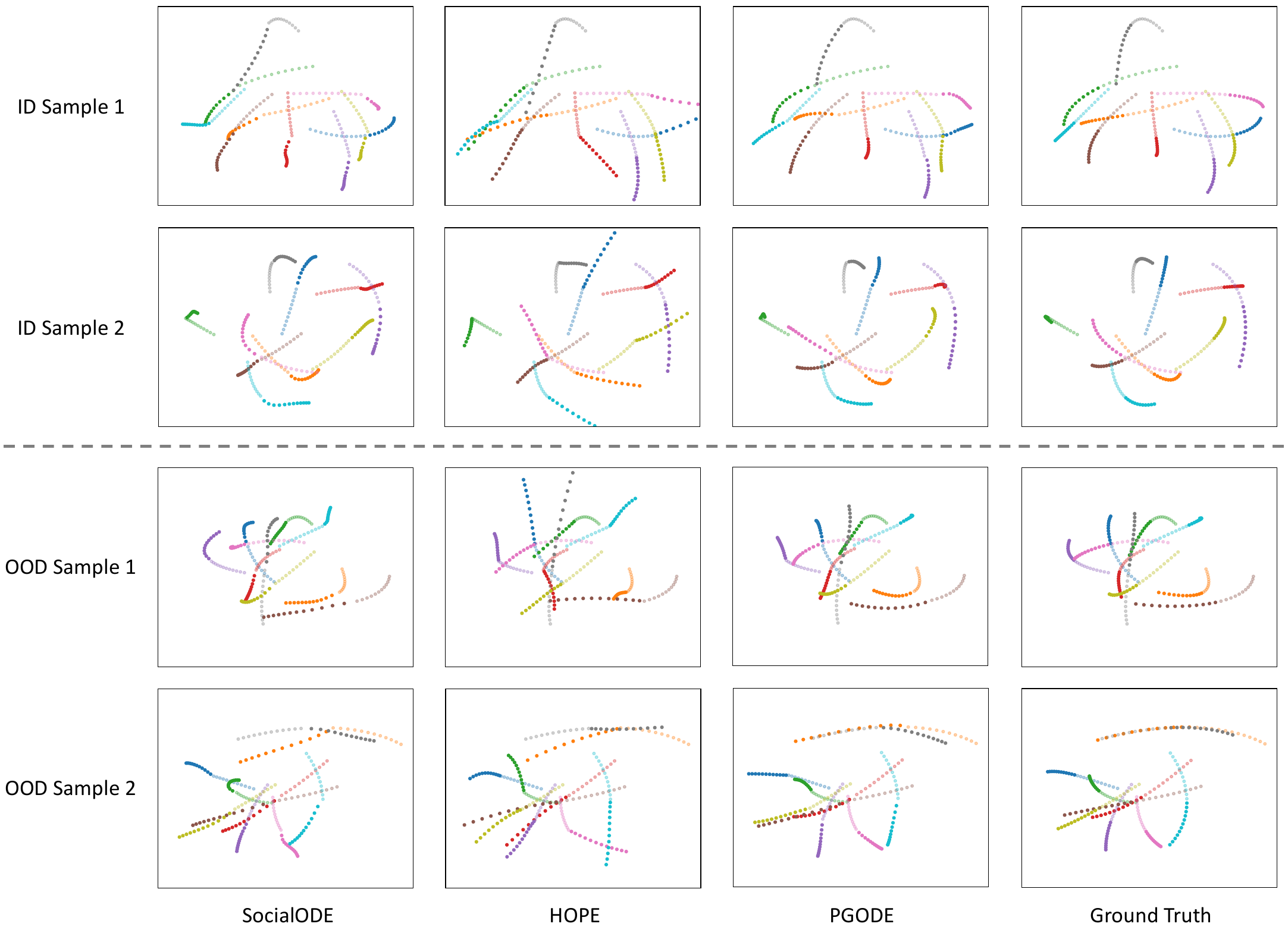}
  \caption{\RRRR{Visualization of different methods on \textit{Springs}. Semi-transparent paths denote observed trajectories and solid paths represent our predictions.}}
  \label{fig:vis_phy_sup}
\end{figure*}

Lastly, we present more visualization of the proposed \method{} and two baselines, i.e., SocialODE and HOPE. We have offered visualization of the predicted trajectory of a sample in Figure \ref{fig:vis_phy} and now we visualize four extra test instances (two ID samples and two OOD samples) in Figure \ref{fig:vis_phy_sup}. From the results, we can observe that the proposed \method{} is capable of generating more reliable trajectories in comparison to the baselines. For instance, our \method{} can discover the correct direction of the orange particle while the others fail in the second OOD instance.

\end{document}